\def\paperversion{2}
\ifnum\paperversion=1
\documentclass[opre,nonblindrev]{informs3}
\fi 
\ifnum\paperversion=2
\documentclass{article}
 \PassOptionsToPackage{numbers,sort&compress}{natbib}
\usepackage{array}
\usepackage[final,nopatch=footnote]{microtype}
\usepackage{graphicx}\usepackage{subcaption} %
\usepackage{booktabs} %
\usepackage{hyperref}
 \usepackage[preprint]{neurips_2025}%\input{arxiv}
 \def\bibfont{\small}%
\usepackage{amsmath}
\usepackage{amssymb}
\usepackage{mathtools}
\usepackage{amsthm}
\fi 

\newtheorem{theorem}{Theorem}[section]
\newtheorem{proposition}[theorem]{Proposition}
\newtheorem{lemma}[theorem]{Lemma}

\theoremstyle{definition}
\newtheorem{definition}[theorem]{Definition}
\newtheorem{assumption}[theorem]{Assumption}
\newtheorem{remark}[theorem]{Remark}

\usepackage{bm}
\usepackage{anyfontsize}
\usepackage{wrapfig}
\usepackage{setspace}

\usepackage[table, svgnames, dvipsnames]{xcolor}
\hypersetup{%
  colorlinks=true,% hyperlinks will be coloured
  linkcolor={blue!66!black},% hyperlink text will be green
  linkbordercolor=red,% hyperlink border will be red
  citecolor={blue!66!black},
}
\usepackage{algorithm}
\usepackage{algorithmic}
\usepackage{mathtools}
\usepackage{booktabs,multirow}
\usepackage{enumitem}
\setlist[enumerate,1]{label=\normalfont{(\Roman*)},leftmargin=*}

\ifnum\paperversion=1
 \def\bibfont{\small}%
\TheoremsNumberedThrough   
\ECRepeatTheorems
\theoremstyle{TH}%
\fi

\DeclareMathAlphabet{\mathsfit}{T1}{\sfdefault}{\mddefault}{\sldefault}
\SetMathAlphabet{\mathsfit}{bold}{T1}{\sfdefault}{\bfdefault}{\sldefault}
\DeclareMathAlphabet{\mathcal}{OMS}{cmsy}{m}{n}

\newcommand{\cP}{\mathcal{P}}

\newcommand{\bR}{\mathbb{R}}
\newcommand{\bE}{\mathbb{E}}
\newcommand{\Reg}{\epsilon}

\newcommand{\diff}{\,\mathrm{d}}

\renewcommand{\emph}[1]{\textit{#1}}

\newcommand{\inp}[2]{\langle #1, #2 \rangle}

\ifnum\paperversion=2
\title{Iterative Sampling Methods for Sinkhorn Distributionally Robust Optimization}
\author{
  Jie Wang \\
  School of Artificial Intelligence, School of Data Science \\
  The Chinese University of Hong Kong, Shenzhen \\
  \texttt{jwang@cuhk.edu.cn} \\
}
\fi
\begin{document}

\maketitle

\begin{abstract}
    Distributionally robust optimization (DRO) has emerged as a powerful paradigm for reliable decision-making under uncertainty. 
    This paper focuses on DRO with ambiguity sets defined via the Sinkhorn discrepancy—an entropy-regularized Wasserstein distance—referred to as Sinkhorn DRO. 
    Existing work primarily addresses Sinkhorn DRO from a dual perspective, leveraging its formulation as a conditional stochastic optimization problem, for which many stochastic gradient methods are applicable. 
    However, the theoretical analyses of such methods often rely on the boundedness of the loss function, and it is indirect to obtain the worst-case distribution associated with Sinkhorn DRO.
    In contrast, we study Sinkhorn DRO from the primal perspective, by reformulating it as a bilevel program with several infinite-dimensional lower-level subproblems over probability space.
    This formulation enables us to simultaneously obtain the optimal robust decision and the worst-case distribution, which is valuable in practical settings, such as generating stress-test scenarios or designing robust learning algorithms.
We propose both double-loop and single-loop sampling-based algorithms with theoretical guarantees to solve this bilevel program. Finally, we demonstrate the effectiveness of our approach through a numerical study on adversarial classification.
\end{abstract}

\section{Introduction}

Distributionally Robust Optimization (DRO) has emerged as a powerful framework for decision-making under uncertainty, aiming to find models that perform well under a set of plausible data distributions, known as the ambiguity set. A prominent approach defines this set using the Wasserstein distance~\citep{esfahani2018data, gao2023distributionally, blanchet2019quantifying, blanchet2021optimal}.
However, the practical application of Wasserstein DRO is limited by two things.
First, the resulting optimization problem is often intractable for general cases, except under some special conditions on the loss function and transportation cost~\citep{liu2021discrete, esfahani2018data, gao2023distributionally, Shafieezadeh15, sinha2018certifiable}, and its worst-case distributions are typically discrete, which can be unrealistic for underlying continuous data and lead to overly conservative decisions.

The Sinkhorn DRO formulation~\citep{wang2021sinkhorn, azizian2023regularization, wang2024regularization, vincent2024texttt, van2025perturbation} addresses these issues by employing an entropy-regularized Wasserstein distance~\citep{cuturi2013sinkhorn}, or called the Sinkhorn discrepancy. This regularization brings two critical benefits. 
First, this problem is tractable for a broad class of loss functions.
Second, it naturally encourages continuous worst-case distributions, often leading to improved generalization.
Due to these advantages, Sinkhorn DRO has found success in various applications such as hypothesis testing~\citep{wang2022data, wang2024non, yang2023distributionally}, experimental design~\citep{jiang2025sinkhorn, dapogny2023entropy}, machine learning~\citep{shen2023wasserstein, songprovably, cescon2025data, ouasfi2025toward}, etc.

Despite its promise, the development of efficient algorithms for Sinkhorn DRO remains an active area of research. Existing methods solve the Sinkhorn DRO from its dual perspective, where the dual objective can be expressed as an expectation of the logarithm of another conditional expectation on the exponential of the risk function. It falls into the class of \textit{conditional stochastic optimization}~(CSO)~\citep{hu2020sample}, and therefore many stochastic gradient methods, such as multi-level Monte-Carlo gradient methods, can be applied to solve it~\citep{hu2020biased, hu2021bias, hu2023contextual, hu2024multi, hu2020sample}.
However, this dual approach has notable limitations. To achieve rate-optimal convergence, existing results require the restrictive assumption that the risk function is bounded~\citep{hu2024multi, wang2021sinkhorn}. Although recent work has relaxed this assumption~\citep{yang2025nested}, it suffers from a sub-optimal convergence rate.

Recently, generative artificial intelligence~(GenAI), especially the score-based generative models~\citep{songscore, song2019generative, lecun2006tutorial, ackley1985learning, lai2025principles}, have drawn much research attention, given their state-of-the-art performance in image and text generation.
These models operate by learning the score function from collected samples, and next sampling from the score function to bridge Gaussian noise to plausible data samples via Langevin dynamics.
In contrast, Sinkhorn DRO aims to learn a worst-case distribution that maximizes a given risk function, rather than choosing a target distribution as a priori.
Based on samples from this worst-case distribution, an optimal robust decision is then trained to minimize the worst-case risk.
Inspired by GenAI, this paper investigates the following question:
\begin{quote}
\textit{
Can we simultaneously obtain the optimal robust decision and corresponding samples from the worst-case distribution in Sinkhorn DRO using a framework similar to that of score-based generative models?
}
\end{quote}

In this paper, we propose a novel primal perspective and iterative sampling algorithms for solving Sinkhorn DRO. %Our approach provides a new algortihmic viewpoint of solving this problem with several advantages.
Our main contributions are summarized as follows.

\paragraph{Reformulating Sinkhorn DRO as Infinite-Dimensional Bilevel Program.}
We consider the soft-constrained Sinkhorn DRO formulation using $n$ samples, where the Sinkhorn discrepancy ball constraint is put in the objective as a penalty term.
We show that this type of formulation can be equivalently formulated as a bilevel optimization with $n$ lower-level subproblems~(See Problem~\eqref{Eq:bilevel:inf}).
Each subproblem seeks the worst-case distribution from the smooth density constructed from the empirical sample point, and the estimated worst-case distribution for Sinkhorn DRO is the average among all worst-case distributions constructed from lower-level problems.
This reformulation inspires us to leverage existing methods from the bilevel program to solve Sinkhorn DRO.

\paragraph{Na\"{\i}ve Double-Loop Algorithm.}
We first develop a double-loop algorithm to solving the bilevel program: at each iteration, we randomly sample a lower-level problem, run an inner loop for a large number of steps to generate the sampling point close to the worst-case distribution of the lower-level problem, and use it to construct a hyper-gradient estimator.
We show that under the smoothness assumption of the loss together with a bounded variance condition, this double-loop algorithm finds a $\varrho$-stationary point of the bilevel program with complexity $\mathcal{O}(\varrho^{-6})$, where the complexity is defined as the number of queries to the gradient of the loss function.

\paragraph{Single-Loop Algorithm.}
Next, we develop a single-loop algorithm by jointly updating the upper-level decision variable and lower-level sampling points.
Specifically, at each iteration, we update several sampling points associated with a randomly-sampled batch of lower-level problems, generate a hypergradient estimator, and employ a momentum update for the upper-level decision variable.
For theoretical analysis, we make an extra assumption that the upper-level hypergradient estimator is constructed using the limit of finitely many sampling points described by a mean-field density.
Then, we show that the iteration complexity becomes $\mathcal{O}(\varrho^{-4})$ for finding $\varrho$-stationary point.
Compared with the double-loop algorithm, this new algorithm only updates $\mathcal{O}(1)$ lower-level problems by one Langevin dynamics step at each iteration, which is computationally and storage efficient.

\paragraph{Numerical Study.}
Finally, we examine our algorithms in numerical study of adversarial classification.
When taking hypergradient estimator as finitely many sampling points for the single-loop algorithm, we still observe convergence of the proposed algorithm.
Also, our single-loop algorithm provides both estimated upper-level decision and estimated sampled points for worst-case distributions of Sinkhorn DRO, which is beneficial for practical high-stake environments that require stress testing.

\subsection{Related Literature}
%\textbf{Algorithms for Solving Sinkhorn DRO.}
During the preparation of this manuscript, we became aware of the concurrent work of~\citet{xu2025gradient}, which also studies sampling-based approaches for general DRO problems and derives convergence rates. 
Our focus and our analysis differ from theirs in the following ways.
First, we presented both double-loop and single-loop algorithms to solve Sinkhorn DRO, whereas they focused only on the double-loop algorithm.
Second, there exists a technical flaw in their analysis, since the constant step size makes the last component of the equation above \citep[Appendix~C4]{xu2025gradient} non-vanishing.
Instead, we provided a rigorous analysis with less restrictive assumptions.

Next, we review papers on several related topics.

\noindent
\textbf{DRO and GenAI.}
Recent work seeks to explore the integration of GenAI with DRO~\citep{xu2024flow, zhu2024distributionally, wen2025distributionally, cheng2025worst}.
Specifically, \citet{xu2024flow} and \citet{cheng2025worst} proposed a flow-based generative model to jointly learn the worst-case distribution and the robust decision for Wasserstein DRO. 
This framework provides an efficient sampler for the worst-case distribution, which is valuable for stress tests and high-stakes environments~\citep{wang2024reliable, wang2022improving}. 
The convergence rate of this framework was later analyzed by \citep{zhu2024distributionally, cheng2025worst}.
In a different approach, \citet{wen2025distributionally} constructed a new DRO framework with an ambiguity set defined using the diffusion model.
While these works integrate generative models into the DRO framework itself, our approach employs Langevin dynamics, a foundational tool for the inference of these models, as a key component for solving the Sinkhorn DRO subproblems.
% Departing from these methods, our work combines Langevin dynamics with Sinkhorn DRO to design novel algorithms, for which we also provide their convergence analysis.

\noindent
\textbf{Bilevel Optimization.}
Sinkhorn DRO can be formulated as a bilevel optimization with multiple lower-level subproblems, which has been studied in recent literature~\citep{guo2021randomized, hu2023blockwise, hu2022multi, hu2023contextual}.
Especially, the complexity of the algorithms in \citep{guo2021randomized, hu2023blockwise, hu2022multi} grows linearly with the number of lower-level problems $n$.
\citet{hu2023contextual} provided a double-loop algorithm that finds the $\varrho$-stationary point with complexity $\mathcal{O}(\varrho^{-4})$, which is independent of $n$ and matches the lower bound for general stochastic optimization. 
However, all these papers assume that the lower-level subproblems are finite-dimensional and strongly convex programs.
% In contrast, we develop methods for solving Sinkhorn DRO from its bilevel reformulation with infinite-dimensional lower-level subproblems.
Our work extends this line to the infinite-dimensional setting, where each lower-level subproblem is an optimization over a probability distribution.
Our framework is closely related to \citet{marion2024implicit} and \citet{xiao2025first}.
These two references only consider bilevel optimization with a single lower-level infinite-dimensional subproblem, which is not applicable to our setting.
In addition, \citet{xiao2025first} provides a double-loop algorithm, which could be computationally and storage inefficient.
We integrate the idea of the single-loop algorithm in \citep{marion2024implicit} and the momentum-based gradient estimator in \citet{hu2022multi} to design our algorithm in Section~\ref{Eq:MF:Single:loop}.

\noindent
\textbf{Infinite-Dimensional Optimization.}
DRO is a special instance of infinite-dimensional optimization. Traditional approaches to these problems can be broadly categorized into three perspectives: discretization, duality, and Frank-Wolfe methods.
First, discretization of the decision space provides a straightforward solution by solving a finite-dimensional tractable approximation problem~\citep{liu2019discrete, chen2019discrete, canuto2005adaptive, garreis2019inexact, ulbrich2017adaptive}. However, this approach suffers from the curse of dimensionality and becomes infeasible in high-dimensional settings.
Second, duality is a powerful approach when the number of constraints is finite, reformulating the primal problem into a finite-dimensional dual via strong duality~\citep{shapiro2001duality, shapiro2021lectures, gao2023distributionally, blanchet2019quantifying, wiesemann2014distributionally, zhen2025unified}. A significant drawback is that the resulting dual problem often contains intractable, infinite-dimensional, or nonconvex subproblems.
Third, Frank-Wolfe methods provide a flexible framework for optimization over probability spaces by iteratively solving linear minimization subproblems~\citep{NEURIPS2021_79121bb9, reddi2016stochastic, jaggi2013revisiting, wang2015functional}. While each iteration involves a finite-dimensional task, this subproblem is often non-convex and high-dimensional, making it difficult to apply without additional problem structure.
A more recent line of work leverages Langevin dynamics for a specific class of problems: entropy-regularized linear optimization over probability distributions~\citep{vempala2019rapid, liuconvergence, cheng2018convergence, dalalyan2017theoretical, durmus2017nonasymptotic}.
This approach efficiently produces sampled points of the optimal distribution, providing a practical pathway for solving a broad class of infinite-dimensional problems. 
This insight also lays the foundation of modern diffusion models. In this work, we adapt and extend the Langevin dynamics framework to design efficient Sinkhorn DRO algorithms.

% With the development of Langevin dynamics, for a special class of optimization, such as optimizing the entropy regularized linear optimization, people have shown that using Langevin dynamics, one can effiectively find its sampling point.
% This observation enables us to solve a class of practical infinite-dimensional optimization, which lays the foundation of the success of diffusion models. 
% In this work, we combine this viewpoint to design efficient Sinkhorn DRO algorithms.

%Infinite-dimensional optimization has been a popular topic in stochastic programming, machine learning, and statistics.
%Prior works have attempted to solve it using several techinques, such as discretizing the decision space, solving the problem from its dual perspective, and frank-wolfe method. 

\subsection{Notations}
For an integer $n\ge1$, we define $[n]=\{1,2,\ldots,n\}$.
We write $\mathbb{E}_{\mu}[f]$ or $\mathbb{E}_{z\sim\mu}[f(z)]$ to denote the expected value of $f(z)$ with respect to $z\sim\mu$.
Denote by $\mathcal{P}(\bR^d)$ the set of probability distributions supported on $\mathbb{R}^d$.
Throughout the paper, we assume the desired accuracy level $\varrho, \delta>0$ is sufficiently small.
We use the notation $\mathcal{O}(\cdot)$ to denote the order in terms of $\varrho$ or $\delta$, and $\widetilde{\mathcal{O}}(\cdot)$ hides terms having the polylog dependency on $\varrho$ or $\delta$.
We provide a list of key mathematical notations and definitions in Appendix~\ref{Appendix:notation}.

\subsection{Organizations}
In Section~\ref{Sec:setup}, we provide a brief review of the existing Sinkhorn DRO results, based on which we reformulate this problem as a bilevel optimization with several infinite-dimensional lower-level problems.
In Section~\ref{Sec:double:loop:naive}, we provide a double-loop iterative sampling algorithm for solving the bilevel program, and analyze its convergence rate.
In Section~\ref{Eq:MF:Single:loop}, we provide a single-loop sampling algorithm for solving the bilevel program, and analyze its convergence rate.
In Section~\ref{Sec:numerical:study}, we provide the numerical study of the distributionally robust classification problem to validate the effectiveness of our proposed algorithms.
In Section~\ref{Sec:conclusion}, we provide several concluding remarks.

%\clearpage
\section{Setup of Sinkhorn DRO}\label{Sec:setup}

We consider the Sinkhorn DRO problem, which takes the form of the minimax optimization:
\begin{equation}
\min_{\theta }\max_{\mu: \mathcal{S}_{\Reg}(\widehat{\mu},\mu)\le\rho}~\bE_{z\sim\mu}[f_{\theta}(z)],\label{Eq:SDRO}
\end{equation}
where $\theta \in \mathbb{R}^{d_{\theta}}$ represents the model parameters, $f_{\theta}(z)$ is the loss function, and $\mathcal{S}_{\Reg}(\widehat{\mu},\cdot)$ is the Sinkhorn Discrepancy (see Definition~\ref{Def:Sinkhorn}) that defines an entropy-regularized ambiguity set around the reference distribution $\widehat{\mu}$.
Following reference~\citep{wang2021sinkhorn}, we provide the definition of Sinkhorn Discrepancy and the strong duality result of Sinkhorn DRO below.
\begin{definition}[Sinkhorn Discrepancy]\label{Def:Sinkhorn}
For regularization parameter $\Reg\ge0$, 
the Sinkhorn Discrepancy between two distributions $\mu$ and $\nu$ is defined as
\[
\mathcal{S}_{\Reg}(\mu,\nu)
=
\inf_{\gamma\in\Gamma(\mu,\nu)}~
\left\{
\frac{1}{2}\bE_{(x,y)\sim\gamma}[\|x-y\|_2^2] + \Reg\bE_{(x,y)\sim\gamma}\left[ 
\log\frac{\diff\gamma(x,y)}{\diff\mu(x)\diff y}
\right]
\right\},
\]
where $\Gamma(\mu,\nu)$ denotes the set of joint distributions with marginal distributions being $\mu$ and $\nu$, respectively.
\end{definition}
We take the reference distribution $\widehat{\mu}$ as an empirical distribution supported on $n$ available data samples $\{x^{(1)},\ldots,x^{(n)}\}$, i.e., $\widehat{\mu}=\frac{1}{n}\sum_{i=1}^n\delta_{x^{(i)}}$.
In this context, Problem~\eqref{Eq:SDRO} can be interpreted as an entropy-regularized approximation of the classical Wasserstein DRO problem with a quadratic cost function. While our focus is on this specific setup, the framework can be extended to general cost functions and reference distributions, which we leave for future work.
The penalized version of Problem~\eqref{Eq:SDRO} admits a strong duality result~\citep{wang2021sinkhorn}, leading to a tractable finite-dimensional reformulation that is crucial for our algorithmic development.
\begin{theorem}[Strong Duality~{\citep{wang2021sinkhorn}}]\label{Thm:strong:duality}
Consider the penalized counterpart of Problem~\eqref{Eq:SDRO}:
\begin{equation}
\min_{\theta}\left\{\max_{\mu\in \mathcal{P}(\mathbb{R}^d)}~\bE_{z\sim\mu}[f_{\theta}(z)]-\lambda\mathcal{S}_{\Reg}( \widehat{\mu},\mu)
\right\}.\label{Eq:SDRO:peanlized}
\end{equation}
Up to an additive constant independent of $\theta$, Problem~\eqref{Eq:SDRO:peanlized} is equivalent to:
\[
\min_{\theta}~\left\{
\frac{\lambda\Reg}{n}
\sum_{i=1}^n\left[ 
\log\bE_{z\sim\mathcal{N}(x^{(i)}, \Reg\mathbf{I}_d)}\big[
e^{f_{\theta}(z)/(\lambda\Reg)}
\big]
\right]
\right\}.
\]
Define the density function
\begin{equation}
u_{\theta,i}(z):=\frac{\diff\mu_{\ast}^{\theta, i}(z)}{\diff z} = \alpha_{i}\cdot \exp\left(
\frac{f_{\theta}(z) - \frac{1}{2}\lambda\|x^{(i)}-z\|_2^2}{\lambda\Reg}
\right),\label{Eq:density:u:theta}
\end{equation}
where $\alpha_i=\Big( 
\int
\exp\left(
\frac{f_{\theta}(z) - \frac{1}{2}\lambda\|x^{(i)}-z\|_2^2}{\lambda\Reg}
\right)\diff z
\Big)^{-1}, i\in[n]$ denotes the normalizing constant.
For fixed $\theta$, the worst-case distribution of Problem~\eqref{Eq:SDRO:peanlized} has density given by
\begin{equation}
\frac{\diff\mu_{\ast}^{\theta}(z)}{\diff z}
=
\frac{1}{n}
\sum_{i=1}^n u_{\theta,i}(z).\label{Eq:density:diff:ast}
\end{equation}
\end{theorem}
\begin{remark}[Soft-Constrained Sinkhorn DRO]
The penalized formulation in \eqref{Eq:SDRO:peanlized} is a Lagrangian counterpart to the constrained original Problem~\eqref{Eq:SDRO}. Here, the radius $\rho$ is replaced by a penalty parameter $\lambda>0$. In practice, achieving good out-of-sample performance requires tuning hyperparameters: either $(\rho, \Reg)$ in the constrained formulation or the equivalent pair $(\lambda, \Reg)$ in the penalized one. 
Our algorithmic development in subsequent sections focuses on solving the penalized formulation~\eqref{Eq:SDRO:peanlized}. The constrained problem can then be addressed via bisection over $\lambda$ (see Algorithm 2 in \citep{wang2021sinkhorn}).
%Furthermore, one can solve the original constrained problem by employing a bisection search on $\lambda$ (see Algorithm 2 in \citep{wang2021sinkhorn}), which involves solving the penalized problem~\eqref{Eq:SDRO:peanlized} multiple times.
\end{remark}

The expression for the worst-case distribution $\mu_{\ast}^{\theta}$ in Theorem~\ref{Thm:strong:duality} reveals that Problem~\eqref{Eq:SDRO:peanlized} is equivalent to a bilevel optimization problem:
\begin{equation}\label{Eq:bilevel:inf}
\begin{aligned}
\min_{\theta}&\quad F(\theta) =
\frac{1}{n}\sum_{i=1}^n\bE_{z\sim \mu_{\ast}^{\theta, i}}[f_{\theta}(z)],\qquad\quad\quad\qquad\mbox{(Upper Level)}\\
&\mu_{\ast}^{\theta, i} = \underset{\mu\in\mathcal{P}(\mathbb{R}^d)}{\arg\min}~\mathcal{D}_{\mathrm{KL}}\Big( 
\mu\Big\| u_{\theta,i}(\cdot)
\Big), \forall i\in[n], \theta,\quad\mbox{(Lower Level)}
\end{aligned}
\end{equation}
where the upper-level problem minimizes the worst-case risk, and for fixed $\theta$, each lower-level problem involves approximating the target distribution with density $u_{\theta,i}(\cdot)$. The lower-level problem is not a standard finite-dimensional optimization, but rather a functional problem over a space of probability distributions. 
Consequently, standard bilevel optimization algorithms for finite-dimensional problems are not directly applicable. To address this, we develop algorithms that leverage the structure of the lower-level problems and combine with sampling techniques inspired by Langevin dynamics. We provide theoretical guarantees for these methods in the following sections.
%Consequently, standard bilevel optimization solvers are not directly applicable. This work addresses this challenge by developing efficient algorithms with theoretical performance guarantees, as detailed in the following sections.

\section{
Na\"{\i}ve Double-Loop Iterative Sampling Algorithm
}\label{Sec:double:loop:naive}

A natural approach for solving the bilevel problem in \eqref{Eq:bilevel:inf} is stochastic gradient descent~(SGD).
The hypergradient of the objective in \eqref{Eq:bilevel:inf} is given by
\begin{equation}
\nabla F(\theta) = \frac{1}{n}\sum_{i=1}^n\bE_{z\sim \mu_*^{\theta,i}}[\nabla_{\theta}f_{\theta}(z)].\label{Eq:F:hyper}
\end{equation}
A direct Monte Carlo estimator for $\nabla F(\theta)$ requires i.i.d. samples $z_i \sim \mu_*^{\theta,i}$ for $i \in [n]$. However, exact sampling from $\mu_*^{\theta,i}$ is intractable. The primary challenge is that $\mu_*^{\theta,i}$ is a complicated distribution with density given in \eqref{Eq:density:u:theta}, and we cannot sample from it exactly, preventing an unbiased gradient estimator. We must therefore rely on biased estimators derived from approximate sampling.
Algorithm~\ref{Alg:solving:SDRO} presents a na\"{\i}ve sampling method by iteratively sampling a point whose distribution is sufficiently close to $\mu_*^{\theta,i}$ and next updating $\theta$ using stochastic gradient descent.

\begin{algorithm}[!htp]
\caption{Na\"{\i}ve Iterative Sampling Algorithm}
\label{Alg:solving:SDRO}
\begin{algorithmic}[1]
\REQUIRE Stepsize parameter $\eta$, initial guess $\theta_0$
\FOR{$k = 0, 1, 2, \dots, T_{\text{out}}-1$}
    \STATE Sample $i_k$ randomly from $\{1,\ldots,n\}$;
    \STATE Draw a sample $z\sim\widetilde{\mu}$ such that $\mathcal{W}_2(\widetilde{\mu}, \mu_{\ast}^{\theta_k, i_k})\le \delta$;
    \STATE 
    Update $\theta_{k+1} = \theta_k - \eta\nabla_{\theta} f_{\theta_k}(z)$.
\ENDFOR\\
\noindent
\textbf{Output }$\widehat{\theta}$ uniformly selected from $\{\theta_1,\ldots,\theta_{T_{\text{out}}}\}$.
\end{algorithmic}
\end{algorithm}

The key step in Algorithm~\ref{Alg:solving:SDRO} is to sample from $\mu_*^{\theta,i}$, whose density function is given in \eqref{Eq:density:u:theta}.
There are many approaches to finish this task, such as Langevin dynamics and its variants~\citep{choi2020metropolis, griffin2013adaptive, cornish2019scalable, haario2001adaptive}, particle-based methods~\citep{liu2017stein, dai2016provable, nitanda2017stochastic},  flow-based methods~\citep{fan2024path, wu2024annealing}, etc.
In this work, we provide theoretical guarantees for Algorithm~\ref{Alg:solving:SDRO} when the sampling in Step 3 is performed using Langevin dynamics, due to its efficiency and well-understood convergence properties.

For fixed $\theta, i\in[n]$, by definition of KL-divergence, the distribution $\mu_*^{\theta,i}$ minimizes the following free energy functional:
\[
\mu_*^{\theta,i} = \underset{\mu\in\cP}{\arg\min}\left\{ 
\bE_{\mu}\left[\frac{-f_{\theta}(z)}{\lambda}+ \frac{1}{2}\|x^{(i)} - z\|_2^2\right] - \Reg\mathcal{H}(\mu) 
\right\}.
\]
A standard method to sample from such a distribution is Langevin dynamics. The corresponding gradient flow is described by the Stochastic Differential Equation (SDE):
\[
\diff Z_t = -\left( \frac{-\nabla_z f_{\theta}(Z_t)}{\lambda} + (Z_t - x^{(i)}) \right)\diff t+ \sqrt{2\Reg}\diff \textbf{W}_t,
\]
where $\{\textbf{W}_t\}_t$ is the standard Brownian motion. Algorithm~\ref{Alg:LSD} presents the discrete-time implementation of this Langevin dynamics. With a sufficiently small step size and a sufficient number of iterations, the output provides an approximate sample from $\mu_*^{\theta,i}$.
\begin{algorithm}[!htp]
\caption{Langevin Stochastic Descent}\label{Alg:LSD}
\begin{algorithmic}[1]
\REQUIRE Stepsize parameter $\tau$, initial distribution $\mu_0$ that is easy to sample.
\STATE Initialize $z_0 \sim \mu_0$
\FOR{$t = 0, 1, 2, \dots, T-1$}
    \STATE Sample $\zeta_t \sim \mathcal{N}(0, \mathbf{I}_d)$
    \STATE Update
    $
    z_{t+1} \leftarrow z_t - \tau\left( \frac{-\nabla_{z} f_{\theta}(z_t)}{\lambda} + (z_t - x^{(i)}) \right) + \sqrt{2\tau\Reg} \zeta_t
    $
\ENDFOR\\
\noindent
\textbf{Output }$z_T$
\end{algorithmic}
\end{algorithm}

\begin{remark}[Comparison with Wasserstein DRO]
If we do not add entropy regularization~(e.g., set $\Reg=0$), 
Algorithm~\ref{Alg:LSD} is the noiseless gradient descent method for unconstrained optimization
\[
\max_{z}\left\{f_{\theta}(z) - \frac{\lambda}{2}\|z - x^{(i)}\|_2^2\right\},
\]
and Algorithm~\ref{Alg:solving:SDRO} reduces to the stochastic gradient descent~(SGD) for optimizing the dual objective of the soft-constrained Wasserstein DRO~\citep{sinha2018certifiable}:
\begin{equation}
\min_{\theta}~\left\{
\frac{1}{n}\sum_{i=1}^n\left[ 
\max_{z}~f_{\theta}(z) - \frac{\lambda}{2}\|z - x^{(i)}\|_2^2
\right]
\right\}.\label{Eq:Lag:WDRO}
\end{equation}
% Thus, the entropy regularization ($\Reg>0$) confers a key computational benefit: it ensures the convergence of the sampling subroutine (Algorithm 2) under milder conditions (e.g., via LSI) compared to the strong concavity required for the inner maximization in Wasserstein DRO.
% See our detailed discussion in subsequent subsections.
This perspective reveals a key computational advantage of Sinkhorn DRO over its unregularized counterpart. The convergence of SGD for the Wasserstein DRO problem \eqref{Eq:Lag:WDRO} typically requires the inner maximization to be strongly concave, which in turn requires a sufficiently large Lagrangian parameter $\lambda$. 
In contrast, as we will show, the entropy regularization ($\Reg>0$) in our setting ensures the sampling algorithm converges without requiring strong concavity, leading to convergence guarantees under milder assumptions.
\end{remark}

\subsection{Convergence Analysis of Algorithm~\ref{Alg:LSD}}
To study the convergence of Algorithm~\ref{Alg:LSD}, we impose the following two assumptions.
\begin{assumption}\label{Assumption:LSD}
\begin{enumerate}
    \item\label{Assumption:LSD:smooth}
For any $(\theta,z)$, the loss function $f_{\theta}(z)$ is continuously differentiable in $z$ and satisfies for any $z,z'$ that $\|\nabla_{z} f_{\theta}(z) - \nabla_{z} f_{\theta}(z')\|_2\le L_{f,2}\|z-z'\|_2.$
\item\label{Assumption:LSD:LSI}
For any $\theta, i\in[n]$, the target distribution $\mu_*^{\theta,i}$ satisfies the log-Sobolev inequality~(LSI) with constant $\alpha>0$.
\end{enumerate}
\end{assumption}

Assumption~\ref{Assumption:LSD}\ref{Assumption:LSD:smooth} is common in the convergence guarantees in optimization.
Assumption~\ref{Assumption:LSD}\ref{Assumption:LSD:LSI} is crucial for establishing the global convergence of Algorithm~\ref{Alg:LSD}, and it holds for a broad class of loss functions.
One sufficient condition of Assumption~\ref{Assumption:LSD}\ref{Assumption:LSD:LSI} is the log-concavity of $\mu_*^{\theta,i}$, i.e., $\lambda$ is sufficiently large such that $z\mapsto f_{\theta}(z) - \frac{\lambda}{2}\|z - x^{(i)}\|_2^2$ is strongly concave. 
In the following proposition, we provide some milder sufficient conditions of Assumption~\ref{Assumption:LSD}\ref{Assumption:LSD:LSI}.
The proof of Proposition~\ref{Proposition:LSI} is provided in Appendix~\ref{Appendix:Sec:double:loop:naive}.
\begin{proposition}[Sufficient Conditions of LSI]\label{Proposition:LSI}
For fixed $\theta$, the following two conditions ensure $\mu_*^{\theta,i}$ satisfies LSI:
\begin{itemize}
\item
Suppose $f_{\theta}(z)$ is bounded such that $\sup f_{\theta}(z) - \inf f_{\theta}(z)<B$ for some constant $B>0$, then $\mu_*^{\theta,i}$ satisfies LSI with constant $\alpha =\frac{1}{\Reg}\exp(-\frac{4B}{\lambda\Reg})$.
\item
Suppose $\|\nabla_z f_{\theta}(z)\|_2\le M$ for any $\theta$ and $z$ with some constant $M>0$, then $\mu_*^{\theta,i}$ satisfies LSI with constant 
\[
\alpha = \frac{1}{2\Reg}
\max\left\{ 
e^{-4M^2/\lambda^2\sqrt{2d/\pi}},
\left( 
4 + (M/\lambda+\sqrt{2})^2
(2 + d + 4M^2/\lambda^2)
e^{M^2/(2\lambda^2)}
\right)^{-1}
\right\}.
\]
\end{itemize}
\end{proposition}
Next, we present the complexity analysis of Algorithm~\ref{Alg:LSD} in the theorem below, following the similar analysis of \citep{vempala2019rapid}.
Its proof is provided in Appendix~\ref{Appendix:Sec:double:loop:naive}.

\begin{theorem}[Complexity of Algorithm~\ref{Alg:LSD}]
\label{Theorem:alg:LSD}
Assume Assumption~\ref{Assumption:LSD} holds.
Specify the parameters in Algorithm~\ref{Alg:LSD} as 
\[
\tau=\frac{\alpha\Reg}{4(1 + L_{f,2}/\lambda)^2}\cdot \min\{1, \delta^2\alpha/(8d)\},\quad 
T = \left\lceil 
\frac{1}{\alpha\tau\Reg}\log\frac{4D_{\mathrm{KL}}(\mu_0\|\mu_*^{\theta,i})}{\delta^2\alpha}
\right\rceil = \widetilde{\mathcal{O}}(\delta^{-2}).
\]
Then, the law of the output of Algorithm~\ref{Alg:LSD}, denoted as $\mu_T$, satisfies that 
$\mathcal{W}_2(\mu_T,\mu_*^{\theta,i})\le \delta$.
\end{theorem}

\begin{remark}[Choice of Initial Distribution]
We recommend taking the initial distribution $\mu_0$ in Algorithm~\ref{Alg:LSD} as $\mathcal{N}(x^{(i)}, \Reg \mathbf{I}_d)$. 
In this case, the KL-divergence $D_{\mathrm{KL}}(\mu_0\|\mu_*^{\theta,i})=\frac{1}{\lambda\Reg}\bE_{\mu_0}[f_{\theta}(z)]$.
Assume $\|\nabla_zf_{\theta}(\tilde{z})\|_2\le M$ for some $\tilde{z}$, then it can be shown that $D_{\mathrm{KL}}(\mu_0\|\mu_*^{\theta,i})=\mathcal{O}(d)$, where $\mathcal{O}(\cdot)$ hides constant related to $\Reg, M$, and $\|x^{(i)}-\tilde{z}\|_2$.
\end{remark}

\subsection{Convergence Analysis of Algorithm~\ref{Alg:solving:SDRO}}

In this subsection, we provide the convergence analysis of Algorithm~\ref{Alg:solving:SDRO}. 
%We report the efficiency of our algorithm using the number of times that we generate the gradient of the loss function $f_{\theta}(z)$ with respect to $\theta$ or with respect to $z$, defined as \emph{computational complexity}.
In our analysis, we measure computational complexity as the total number of gradient evaluations of $f_{\theta}(z)$ (with respect to either $\theta$ or $z$). This metric dominates the total computational time from Algorithms~\ref{Alg:solving:SDRO} and~\ref{Alg:LSD}. We further impose the following assumptions.

%From our algorithm design, it is easy to check that the total computational time is roughly proportional to the computational complexity.

\begin{assumption}\label{Assumption:SDRO}
\begin{enumerate}
\item\label{Assumption:SDRO:I}
For any $\theta$, the distribution $\mu_*^{\theta,i}, i\in[n]$ satisfies that 
\begin{equation}
\mathbb{V}\mathrm{ar}_{(i,z_i)\sim \mathrm{Uniform}([n])\otimes \mu_*^{i,\theta}}\big(\nabla f_{\theta}(z_i)\big)\le \sigma^2\label{Eq:variance:mu:z1i}
\end{equation}
for some constant $\sigma^2>0$.
\item\label{Assumption:SDRO:II}
For any $(\theta,z)$, the loss function $f_{\theta}(z)$ is continuously differentiable in $\theta$ and satisfies for any $z,z'$ that 
\begin{align*}
\|\nabla f_{\theta}(z)\|_2&\le L_{f,1},\\
\|\nabla_{\theta}f_{\theta}(z) - \nabla_{\theta}f_{\theta}(z')\|_2&\le L_{f,2}\|z -z'\|_2,\\
\|\nabla_{\theta}f_{\theta}(z) - \nabla_{\theta'}f_{\theta}(z)\|_2&\le L_{f,2}\|\theta - \theta'\|_2.
\end{align*}
\item
For any $(\theta,z)$, the loss function $f_{\theta}(z)$ is continuously differentiable in $z$ and satisfies for any $\theta,\theta'$ that $\|\nabla_{z} f_{\theta}(z) - \nabla_{z} f_{\theta'}(z)\|_2\le L_{f,2}\|\theta-\theta'\|_2.$
\end{enumerate}
\end{assumption}
Recall that in Algorithm~\ref{Alg:solving:SDRO}, we constructed the hypergradient estimator of \eqref{Eq:F:hyper} as 
\begin{equation}\label{Eq:F:hyper:estimate}
\widehat{\nabla}F(\theta; z) = \nabla_{\theta}f_{\theta}(z),
\end{equation}
where the random vector $z$ follows the distribution $\widetilde{\mu}$ with $\mathcal{W}_2(\widetilde{\mu}, \mu_{\ast}^{\theta, i})\le \delta$ and $i$ is a random sample from $[n]$.
The following lemma provides bias, variance, and computational complexity analysis regarding our estimator.
Its proof is provided in Appendix~\ref{Appendix:Sec:double:loop:naive}.
\begin{lemma}[Bias, variance, and complexity of hypergradient estimator]\label{Lemma:hyperparameter}
Assume Assumptions~\ref{Assumption:LSD} and \ref{Assumption:SDRO} hold, then the estimator in \eqref{Eq:F:hyper:estimate} satisfies that
\begin{enumerate}
    \item(Bias)
$\left\|\bE\big[\widehat{\nabla}F(\theta; z)\big]
-
{\nabla}F(\theta)
\right\|_2
\le L_{f,2}\delta.
$
    \item\label{Lemma:hyperparameter:II}(Variance)
$\mathbb{V}\mathrm{ar}(\widehat{\nabla}F(\theta; z))\le \mathbf{V}:=2\sigma^2 + 2L_{f,1}^2\sqrt{\alpha}\delta + 2L_{f,2}^2\delta^2$.
    \item(Complexity)
The computational complexity of constructing \eqref{Eq:F:hyper:estimate} is $\widetilde{\mathcal{O}}(\delta^{-2})$.
\end{enumerate}
\end{lemma}

A random vector $\theta$ is said to be a $\varrho$-stationary point if $\mathbb{E}\|\nabla F(\theta)\|_2^2\le \varrho^2$.
As the bilevel program~\eqref{Eq:bilevel:inf} is nonconvex in general, we focus on finding its stationary point using Algorithm~\ref{Alg:solving:SDRO}.
We provide its computational complexity in the theorem below.
Its proof is provided in Appendix~\ref{Appendix:Sec:double:loop:naive}.

\begin{theorem}[Complexity Bounds]\label{Theorem:complexity:bound}
Assume Assumptions~\ref{Assumption:LSD} and \ref{Assumption:SDRO} hold.
In order to use Algorithm~\ref{Alg:solving:SDRO} to obtain a $\varrho$-stationary point, it suffices to specify parameters in Algorithm~\ref{Alg:solving:SDRO} as
\[
T_{\text{out}} = \mathcal{O}(\varrho^{-4}), \quad \eta=\frac{1}{\sqrt{T_{\text{out}}\mathbf{V}}}, \quad\delta = \frac{\varrho}{2L_{f,2}},
\]
where $\mathcal{O}(\cdot)$ hides constant depending only on the initial guess, $\mathbf{V}$, and $L_{f,2}$.
Consequently, the total computational complexity of Algorithm~\ref{Alg:solving:SDRO} to find a $\varrho$-stationary point is $\widetilde{\mathcal{O}}(\varrho^{-6})$.
%The total computational complexity of Algorithm~\ref{Alg:solving:SDRO} to obtain a $\varrho$-stationary point is $\mathcal{O}(\varrho^{-6})$.
\end{theorem}

\begin{remark}[Comparison with \citep{xu2025gradient}]
In this section, we present a double-loop algorithm similar to \citep{xu2025gradient}.
We remark that there are some differences in the convergence analysis part.
First, the complexity result in \citep[Theorem~2]{xu2025gradient} uses a constant stepsize, which leads to a non-vanishing optimization error (bias) in the gradient estimator~(see the last component of the equation above \citep[Appendix~C4]{xu2025gradient}). In contrast, our decaying stepsize schedule ensures this error vanishes asymptotically.
% Second, we assume the bounded variance condition \eqref{Eq:variance:mu:z1i} holds for $\mathrm{Uniform}([n])\otimes\mu_*^{i,\theta}$, wheras their result assumes the uniform bounded variance condition for any generated stochastic gradient estimators, which may need further justification (such as Lemma~\ref{Lemma:hyperparameter}\ref{Lemma:hyperparameter:II}).
Second, our variance assumption (Assumption \ref{Assumption:SDRO}\ref{Assumption:SDRO:I}) is made on the ideal gradient at the true worst-case distributions $\mu_*^{\theta,i}$. We then explicitly control how the sampling error propagates to the variance of our practical estimator (Lemma \ref{Lemma:hyperparameter}\ref{Lemma:hyperparameter:II}). Their analysis assumes a uniform bound on the variance of the approximate stochastic gradients, which is a stronger condition.
\end{remark}

%\clearpage
\section{Mean-Field Single-Loop Sampling Algorithm}
\label{Eq:MF:Single:loop}

The double-loop algorithm (Algorithm~\ref{Alg:solving:SDRO}) is conceptually simple but can be computationally expensive, as each outer update requires running an inner loop of Langevin dynamics to high accuracy. To improve efficiency, we now propose a single-loop algorithm that interleaves the updates of the upper-level parameter and the lower-level distribution estimators.

For $i\in[n]$, our algorithm initializes the worst-case distribution estimators $\mu_0^{(i)} = \mathcal{N}(x^{(i)}, \Reg\mathbf{I}_d)$ for the $i$-th lower-level subproblem.
At the beginning of each iteration $k$, we sample a set of lower-level subproblems with index $i\in I_k$.
Next, we update the distribution estimator $\mu_{k+1}^{(i)}$ such that each particle is updated using one-step of Langevin dynamics:
\begin{equation}\label{Eq:update:mu:k+1}
\begin{aligned}
z_{k+1}^{(i)} &= \left\{ 
\begin{aligned}
z_k^{(i)} - \tau\left( 
    -\frac{\nabla_zf_{\theta_k}(z_k^{(i)})}{\lambda} + (z_k^{(i)} - x^{(i)})
    \right) + \sqrt{2\tau\Reg}\zeta_k^{(i)},&\quad \text{if }i\in I_k\\
z_k^{(i)},&\quad\text{otherwise}.
\end{aligned}
\right.\\
\mu_{k+1}^{(i)} &=\mathrm{Law}\big(z_{k+1}^{(i)}\big).
\end{aligned}
\end{equation}
We then update the gradient estimator of the upper-level objective as
\begin{equation}\label{Eq:update:v:k+1}
v_{k+1}=\frac{1}{|I_k|}\sum_{i\in I_k}\bE_{z_{k}^{(i)}\sim\mu_{k}^{(i)}}\big[\nabla_{\theta}f_{\theta_k}(z_{k}^{(i)})\big].
\end{equation}
Compared with the true gradient $\nabla F(\theta_k)$ defined in \eqref{Eq:F:hyper}, we replace the average over all indices $i\in[n]$ with the average over sampled indices $i\in[I_k]$, and replace the true worst-case distribution $\mu_*^{\theta_k,i}$ with its estimator $\mu_k^{(i)}$ obtained from the last iteration.
Finally, we maintain the moving average estimator $r_{k+1}$ for the gradient estimator $v_{k+1}$ and update the upper-level decision $\theta_{k+1}$ using $r_{k+1}$.
The detailed procedure is provided in Algorithm~\ref{Alg:solving:SDRO:single:improved}.

\begin{remark}[Mean-Field Update]
We call Algorithm~\ref{Alg:solving:SDRO:single:improved} a mean-field algorithm because the gradient estimator $v_{k+1}$ in \eqref{Eq:update:v:k+1} uses the population expectation $\bE_{z_k^{(i)}\sim\mu_k^{(i)}}[\cdot]$, corresponding to the limit of an infinite number of particles. 
In practice, this expectation is approximated by a sample average over a finite set of particles.
%It is noteworthy that the estimator $v_{k+1}$ in Algorithm~\ref{Alg:solving:SDRO:single:improved} is generated using population expectation, and that is the reason we name the algorithm as mean-field single-loop iterative sampling algorithm. 
%For practical implementation, we approximate the update in Step~4 and Step~5 using finite particles. 
The technical difficulty of the analysis for the finite-particle case is that this will make the estimator $v_{k+1}$ stochastic provided that $I_k$ is fixed.
A possible solution is the propagation of chaos~\citep{nitanda2024improved, chaintron2022propagation, suzuki2023uniform, zhu2024mean, mei2018mean}, 
a theory that quantifies the difference between the dynamics of a system of finitely many particles and its limiting behavior described by their mean-field densities.
In this work, our convergence analysis (Theorem~\ref{Theorem:convergence:final}) focuses on the idealized mean-field dynamics (population expectations). Analyzing the finite-particle case via propagation of chaos is an important direction for future work.
\end{remark}

\begin{algorithm}[!htp]
\caption{Mean-Field Single-Loop Iterative Sampling Algorithm}
\label{Alg:solving:SDRO:single:improved}
\begin{algorithmic}[1]
\REQUIRE Stepsize parameters $\eta, \tau$, initial guess $\mu_0^{(i)}, i\in[n]$, moving average estimator $r_0$, moving average parameter $\beta_0$, number of iterations $T$, mini-batch size $|I_k|$.
\FOR{$k = 0, 1, 2, \dots, T-1$}
\STATE{Randomly sample indices $I_k\subseteq [n]$}
\FOR{$i\in[n]$}
\STATE{Update $\mu_{k+1}^{(i)}$ according to \eqref{Eq:update:mu:k+1}
}
\ENDFOR
    \STATE Update gradient estimator $v_{k+1}$ according to \eqref{Eq:update:v:k+1}
    \STATE{Update $r_{k+1} = (1-\beta_0)r_k + \beta_0v_{k+1}$}
    \STATE 
    Update $\theta_{k+1} = \theta_k - \tau\eta r_{k+1}$.
\ENDFOR\\
\noindent
\textbf{Output }$\widehat{\theta}$ uniformly selected from $\{\theta_1,\ldots,\theta_{T}\}$.
\end{algorithmic}
\end{algorithm}

\subsection{Convergence Analysis of Algorithm~\ref{Alg:solving:SDRO:single:improved}}\label{Sec:proof:cov:ana}
In this subsection, We outline the key steps of the convergence analysis; detailed proofs are deferred to Appendix \ref{Appendix:Sec:proof:cov:ana}.
We first build the error bound for the objective at $k$-th and $(k+1)$-th iterations in the lemma below.
\begin{lemma}[Descent Lemma]\label{Lemma:descent}
Assume Assumption~\ref{Assumption:SDRO}\ref{Assumption:SDRO:II} holds and the stepsize parameters satisfy $\eta\tau\le\frac{1}{2L_{f,2}}$.
Consider the update in Step~8 of Algorithm~\ref{Alg:solving:SDRO:single:improved}, it holds that 
     \[
     F(\theta_{k+1}
     )\le F(\theta_k) + \frac{\tau\eta}{2}\|\nabla F(\theta_k) - r_{k+1}\|_2^2
     -
     \frac{\tau\eta}{2}\|\nabla F(\theta_k)\|_2^2
     -
     \frac{\tau\eta}{4}\|r_{k+1}\|_2^2.
     \]
\end{lemma}
The critical step for our analysis is to bound the difference between the true hypergradient $\nabla F(\theta_k)$ and its estimator $r_{k+1}$.
To this end, let us define the auxiliary gradient
\begin{equation}
\nabla F(\theta_k;{\mu}_k^{(1:n)})
:=
\frac{1}{n}\sum_{i\in[n]}\bE_{z_k^{(i)}\sim{\mu}_k^{(i)}}[\nabla f_{\theta}(z_k^{(i)})].
\label{Eq:F:aux}
\end{equation}
Denote by $\mathcal{G}_k$ the $\sigma$-algebra generated by all the randomness up to iteration $k$.
Conditioned on $\mathcal{G}_{k-1}$, it can be shown that $v_{k+1}$ is the unbiased estimator of $\nabla F(\theta_k;\tilde{\mu}_k^{(1:n)})$, as the former uses mini-batch random samples $I_k\subseteq[n]$.
In the following, we provide the upper bound of the difference between $\nabla F(\theta_k)$ and $r_{k+1}$ using $\nabla F(\theta_k;{\mu}_k^{(1:n)})$ and $v_{k+1}$.
\begin{lemma}[Gradient Difference Lemma]\label{Lemma:gradient:diff}
Assume Assumption~\ref{Assumption:SDRO}\ref{Assumption:SDRO:II} holds, and the parameter $\beta_0\in(0,1]$, then it holds that
\begin{equation}
\begin{multlined}
\bE\big[\|\nabla F(\theta_k) - r_{k+1}\|_2^2\big|\mathcal{G}_{k-1}\big]
\le
(1-\beta_0)\|\nabla F(\theta_{k-1}) - r_{k}\|_2^2 
+ 
\frac{4L_{f,2}^2\tau^2\eta^2}{\beta_0}\|r_k\|_2^2 \\+ 
4\beta_0\left\| \nabla F(\theta_k) - \nabla F(\theta_k; \mu_{k}^{(1:n)})\right\|_2^2 
+ \beta_0^2\bE\left[\left\|\nabla F(\theta_k; \mu_{k}^{(1:n)}) - v_{k+1}\right\|_2^2\middle|\mathcal{G}_{k-1}\right].
\end{multlined}\label{Eq:gradient:diff:lemma}
\end{equation}
\end{lemma}

A key technical challenge arises in bounding the third term of \eqref{Eq:gradient:diff:lemma}. By Pinsker’s Inequality~\citep{cover1999elements}, 
\begin{align*}
&\left\| \nabla F(\theta_k) - \nabla F(\theta_k; \mu_{k}^{(1:n)})\right\|_2^2
\\=&
\left\|
\frac{1}{n}\sum_{i\in[n]}\left[ 
\bE_{z\sim\mu_k^{(i)}}[\nabla f_{\theta_k}(z)] - \bE_{z\sim\mu_*^{\theta_k,i}}[\nabla f_{\theta_k}(z)]
\right]
\right\|_2^2\\
\le&
\frac{L_{f,1}^2}{n}\sum_{i\in[n]}
\mathrm{TV}(\mu_k^{(i)}, \mu_*^{\theta_k,i})^2\le 
\frac{L_{f,1}^2}{2n}\sum_{i\in[n]}\mathcal{D}_{\mathrm{KL}}(\mu_{k}^{(i)}, \mu_*^{i, \theta_k}).
%\bE_{z\sim\mu_k^{(i)}}[\nabla f_{\theta_k}(z)] - \bE_{z\sim\mu_*^{\theta_k,i}}[\nabla f_{\theta_k}(z)]
\end{align*}
%$|\nabla F(\theta_k) - \nabla F(\theta_k; \mu_k^{(1:n)})|2^2$. 
This error depends on the KL-divergence between our running estimators $\mu_k^{(i)}$ and the true optimal distributions $\mu*^{i,\theta_k}$. Unlike standard bilevel optimization where lower-level solutions lie in a Euclidean space, here they are probability distributions. Therefore, standard techniques~(such as \citep{hu2022multi, qiu2022large}) for Euclidean norm bounds, especially the triangular inequality, do not apply.
% On the right-hand-side of \eqref{Eq:gradient:diff:lemma}, it can be shown that the last component can be bounded in terms of the variance of $v_{k+1}$, and the third component can be bounded as
% \[
% \left\| \nabla F(\theta_k) - \nabla F(\theta_k; \mu_{k}^{(1:n)})\right\|_2^2
% \le 
% \frac{L_{f,1}^2}{n}\sum_{i\in[n]}\mathcal{D}_{\mathrm{KL}}(\mu_{k}^{(i)}, \mu_*^{i, \theta_k}).
% \]
% In the above equation, the distribution $\mu_k^{(i)}$ denotes the estimated distribution at $k$-th iteration and the distribution $\mu_*^{i,\theta_k}$ denotes the optimal lower-level solution with fixed upper-level decision $\theta\equiv\theta_k$.
% Unlike the conventional techniques in bilevel optimization literature~(such as \citep{hu2022multi, qiu2022large}), one cannot leverage the triangular inequality in finite-dimensional Euclidean space to obtain the desired error bound on the KL-divergence term.
To overcome this, we adapt and extend the SDE-based techniques from sampling literature \citep{vempala2019rapid, marion2024implicit} to our setup, which contains multiple target distributions, and each distribution is time-varying.
Lemma \ref{Lemma:KL:div} provides the desired error bound, relating the cumulative KL error to algorithm parameters and the momentum gradient norm $\|r_{k+1}\|_2^2$.
%whose target distribution is time-varying, 
%Instead, we adopt the stochastic differential equation~(SDE)-based approach in \citep{vempala2019rapid, marion2024implicit} to provide its upper bound.

\begin{lemma}[KL-Divergence Bound]\label{Lemma:KL:div}
Assume Assumptions~\ref{Assumption:LSD} and \ref{Assumption:SDRO} hold, then it holds that
\begin{align*}
&\sum_{k=0}^T\bE\big[ 
\mathcal{D}_{\mathrm{KL}}(\mu_{k}^{(i)} \| \mu_*^{i, \theta_{k}})
\big]\\
\le&
\frac{4n}{\alpha\tau|I_k|}
\mathcal{D}_{\mathrm{KL}}(\mu_{0}^{(i)} \| \mu_*^{i, \theta_{0}})
 + \frac{12\eta L_{f,1}^2Tn}{\lambda\Reg\alpha|I_k|} + \frac{32\tau\Reg dL_{G,2}^2T}{\alpha} + \frac{20\eta n}{\lambda\Reg\alpha|I_k|}\sum_{k=0}^{T-1}\bE\big[ \|r_{k+1}\|_2^2
\big].
\end{align*}
\end{lemma}

Combining the lemmas above, we derive the convergence theorem for Algorithm~\ref{Alg:solving:SDRO:single:improved} below.

\begin{theorem}[Convergence Analysis of Algorithm~\ref{Alg:solving:SDRO:single:improved}]\label{Theorem:convergence:final}
Assume Assumptions~\ref{Assumption:LSD} and \ref{Assumption:SDRO} hold, and with the following choices of parameters in Algorithm~\ref{Alg:solving:SDRO:single:improved}:
\begin{align*}
\beta_0&\le \frac{\varrho^2|I_k|}{6L_{f,2}^2},\quad 
\tau\le \frac{\varrho^2\alpha}{384\Reg dL_{G,2}^2L_{f,1}^2},\\
\eta&\le \min\left( 
\frac{\varrho^2\lambda\Reg\alpha|I_k|}{144L_{f,1}^2L_{f,2}^2n}, 
\frac{\lambda\Reg\alpha|I_k|}{160L_{f,2}^2n}
\right)\\
T&\ge \max\left( 
\frac{12\bE[F(\theta_0) - F(\theta_*)]}{\eta\tau\varrho^2}, 
\frac{6\bE\|\nabla F(\theta_0) - z_1\|_2^2}{\beta_0\varrho^2}, 
\frac{48L_{f,2}^2}{\alpha\tau|I_k|\varrho^2}\sum_{i=1}^n\mathcal{D}_{\mathrm{KL}}(\mu_0^{(i)}\|\mu_*^{i,\theta_0})
\right),
\end{align*}
Algorithm~\ref{Alg:solving:SDRO:single:improved} finds a $\varrho$-stationary solution of Problem~\eqref{Eq:bilevel:inf}.
The total computational complexity of Algorithm~\ref{Alg:solving:SDRO} to obtain a $\varrho$-stationary point is $\mathcal{O}(\varrho^{-6}\cdot \frac{n}{|I_k|})$, where $\mathcal{O}(\cdot)$ hides constants depending only on $\lambda,\Reg,\alpha,L_{f,1},L_{f,2}$, and the initial guess.
\end{theorem}

%\begin{remark}
In Theorem~\ref{Theorem:convergence:final}, the obtained $\varrho^{-6}$ complexity is likely suboptimal compared to the $\varrho^{-4}$ lower bound for general nonconvex stochastic optimization. We hypothesize this gap stems from two sources: (i) the discretization error of the Langevin step, which requires careful control of the KL divergence between iterated distributions, and (ii) the conservative nature of our KL-divergence bound (Lemma \ref{Lemma:KL:div}) compared to error bounds typically available for finite-dimensional strongly convex subproblems.
%we obtain the complexity that is suboptimal compared with the lower bound $\mathcal{\varrho}^{-4}$ for solving general nonconvex stochastic optimization problems.
%A possible explanation is that the discrete-time Langevin dynamics step in Step~4 of Algorithm~\ref{Alg:solving:SDRO:single:improved} induces slower convergence rate.
%Therefore, the KL-divergence bound in \eqref{Lemma:KL:div} is conservative compared with the error bound of lower-level solutions obtained in the finite-dimensional Euclidean space.
Also, in our complexity analysis, we do not specify the mini-batch size $|I_k|$ in each iteration.
This offers a trade-off: using $|I_k|=1$ minimizes per-iteration cost, while larger $|I_k|$ reduces the number of iterations $T$ needed, which can be exploited in parallel computing settings

%It can be as small as one, or in parallel computing environment, one can reduce the iteration complexity by increasing the batch size $|I_k|$ in each iteration.
%\end{remark}

%\clearpage
\begin{figure}[ht]
    \centering
    \begin{subfigure}{\textwidth}
        \centering
        \includegraphics[width=\textwidth]{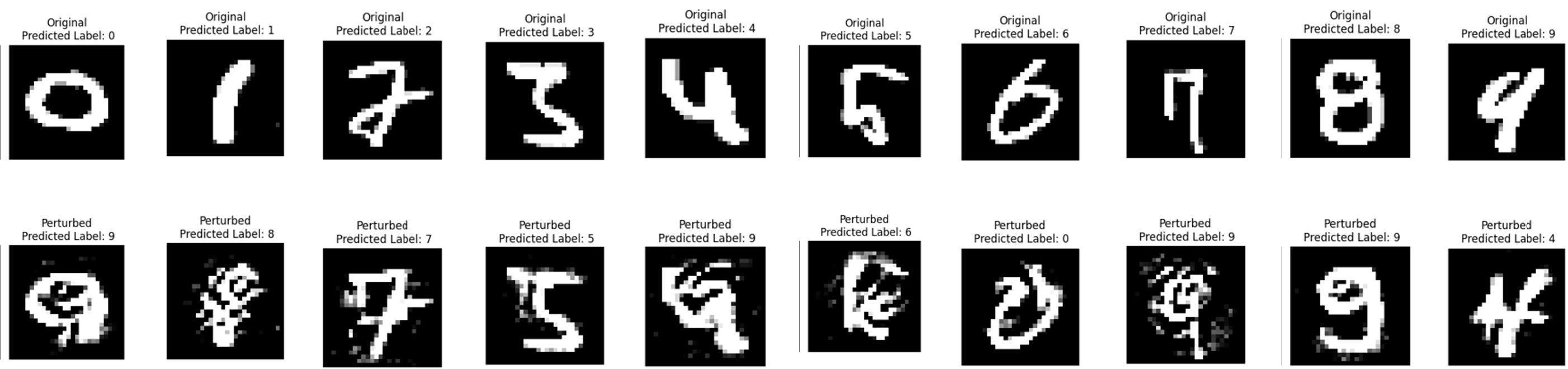}
        \caption{MNIST Dataset}
        \label{fig:raw_samples}
    \end{subfigure}
    \hfill
    \begin{subfigure}{\textwidth}
        \centering
        \includegraphics[width=\textwidth]{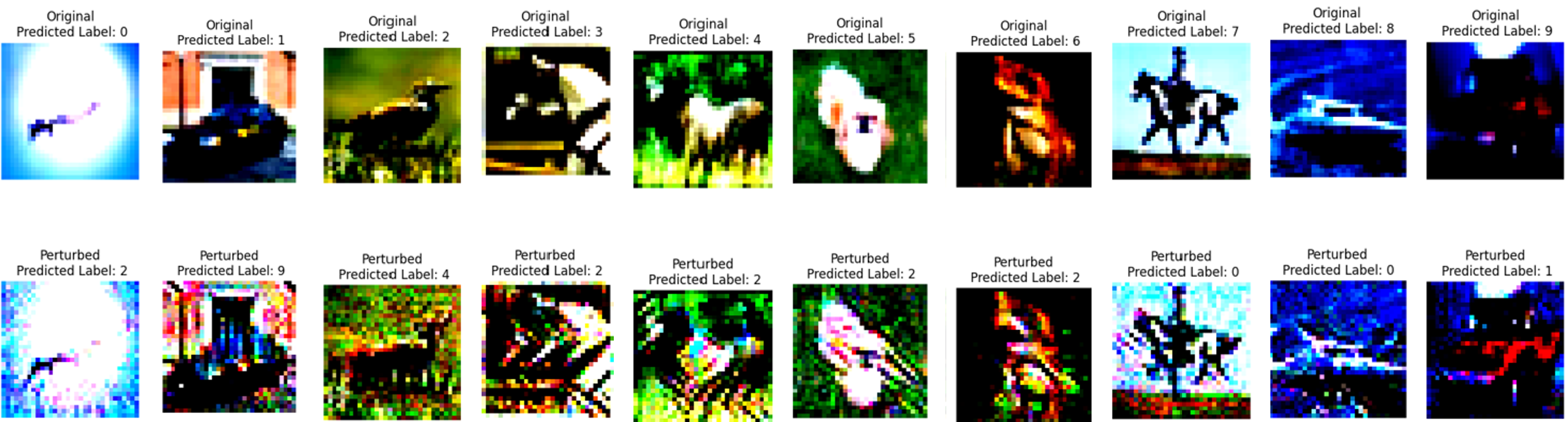}
        \caption{CIFAR-10 Dataset}
        \label{fig:adversarial_samples}
    \end{subfigure}
    \caption{Raw and adversarial samples found by the Sinkhorn DRO. 
    Two subfigures represent numerical experiments for two datasets (MNIST and CIFAR-10).
    For each subfigure, plots on the top represent the raw samples, and plots on the bottom represent the perturbed samples using Algorithm~\ref{Alg:solving:SDRO:single:improved}.}
    \label{fig:mnist_comparison}
\end{figure}
\section{Applications}\label{Sec:numerical:study}
In this section, we validate the performance of our proposed algorithms for the adversarial classification task and follow the similar experiment setup as in \citep{sinha2018certifiable}. It aims to solve the distributionally robust optimization problem
\[
\min_{\theta}~\left\{\max_{\mu}~\bE_{(x,y)\sim\mu}[\ell(f_{\theta}(x), y)] - \lambda\mathcal{S}_{\Reg}(\widehat{\mu}, \mu)\right\},
\]
where $f_{\theta}(x)$ denotes a neural network classifier consisting of $8*8, 6*6,$ and $5*5$ convolutional filer layers with ELU activations followed by a fully connected layer and softmax output, $\ell(\hat{y}, y)$ denotes the cross-entropy classification loss, and the Sinkhorn discrepancy peanlty only considers the perturbation for the data feature part instead of the data label part.
The reference distribution $\widehat{\mu}$ is constructed using the empirical samples from the MNIST~\citep{lecun10} or CIFAR-10~\citep{krizhevsky2009learning} training dataset.

\subsection{Visualization of Worst-Case Distributions}
In this subsection, we visualize the samples from worst-case distribution by solving the bilevel program~\eqref{Eq:bilevel:inf} using Algorithm~\ref{Alg:solving:SDRO:single:improved}.
We specify the hyper-parameters $\Reg=0.1$ and $\lambda=20$.
The results are provided in Figure~\ref{fig:mnist_comparison}, which shows the qualitative changes from the original images to perturbed images for the MNIST or CIFAR-10 datasets. 
The figures demonstrate that Sinkhorn DRO induces more meaningful contextual changes to the original images to confuse the classifier, which is more aligned with our intuition for the adversarial classification task.

\subsection{Ablation Study}
Next, we quantitatively validate the effectiveness of Algorithm~\ref{Alg:solving:SDRO:single:improved}.
%We use the same neural network architecture as in the last subsection.
We fine-tune the regularization for Sinkhorn DRO or Wasserstein DRO such that the $\ell_2$-norm between the original image $X$ and the perturbed one is controlled within $C_2 = 0.04*\mathbb{E}_{\widehat{\mu}}[\|X\|_2]$.
To assess the robustness of the proposed models, we apply a Projected Gradient Method (PGM) attack with $\ell_2$-norm constraints to the test datasets. The perturbation magnitude $\Delta$ is normalized by the average $\ell_2$-norm of the test features. 
We examine the performance using the misclassification rate on perturbed datasets.

%Model performance under adversarial perturbations is quantified by the misclassification rate.

% We utilize the MNIST and CIFAR-10 datasets to validate the effectiveness of our proposed algorithm\ref{Alg:solving:SDRO:single:improved}. For the MNIST dataset, we train a compact neural network classifier consisting of three convolutional layers with filter sizes $8 \times 8$, $6 \times 6$, and $5 \times 5$, respectively, each followed by an ELU activation function. This convolutional feature extractor is cascaded with a fully connected layer and a softmax output layer to generate class probabilities. 

% For the CIFAR-10 dataset, we adopt a transfer learning strategy by leveraging a ResNet-50 network pre-trained on the ImageNet dataset as the backbone feature extractor. To adapt this pre-trained model to the CIFAR-10 classification task, we append a simple custom classification head comprising two fully connected layers (with dimensionality transitions 2048→512 and 512→10), interspersed with a ReLU activation function and a dropout layer for regularization. During training, we unfreeze the last two layers of the pre-trained ResNet-50 backbone to enable fine-tuning on the CIFAR-10 data. $C_2 = \mathbb{E}_{\hat{P_n}}[\|X\|_2]$. For the WDRO framework, we use a regularization term defined as $\gamma = 0.04*C_2$.

\begin{figure}[!htbp]
    \centering
    \begin{subfigure}{\textwidth}
        \centering
        \includegraphics[width=0.8\textwidth]{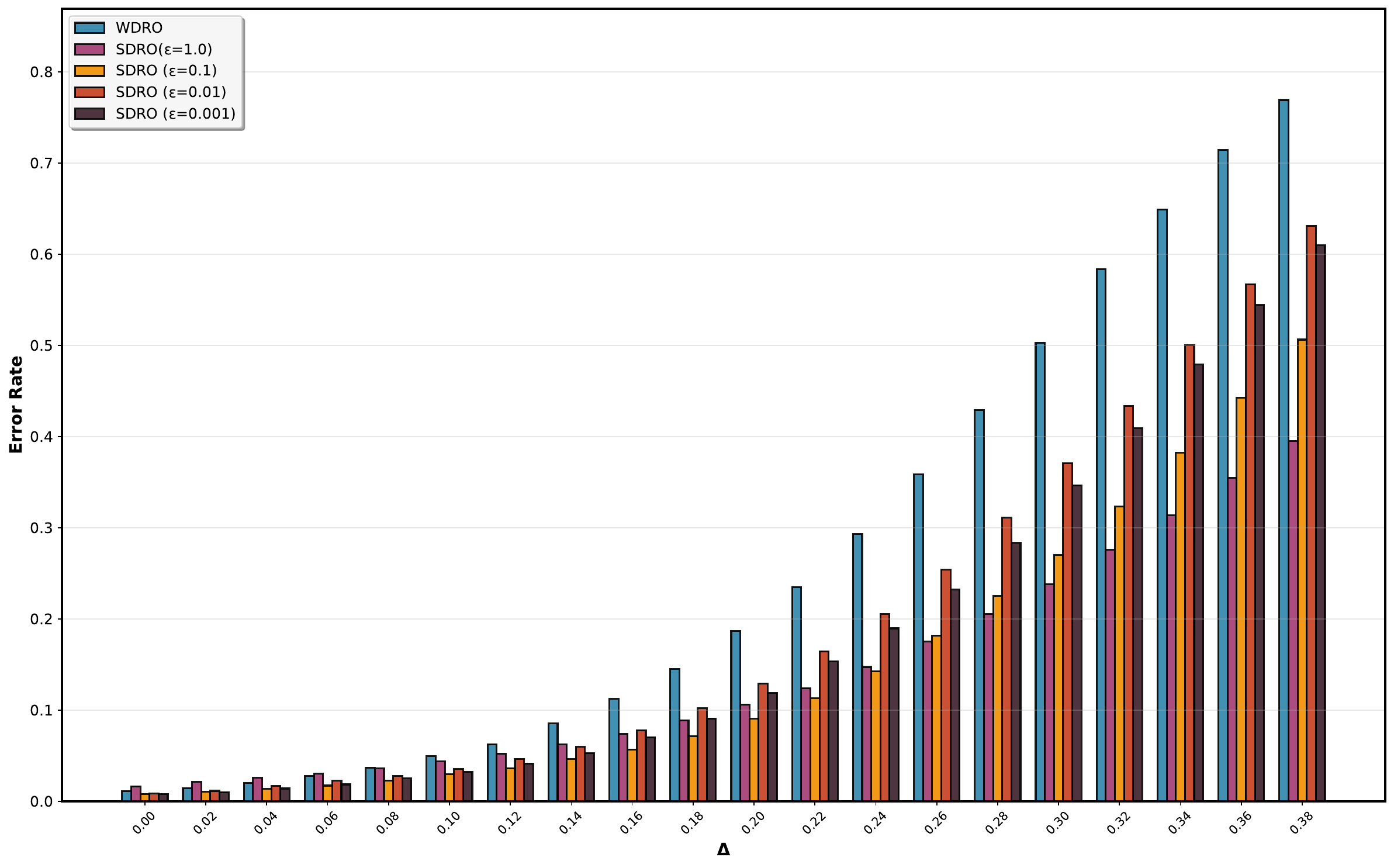}
        \caption{MNIST Dataset}
        \label{fig:mnist}
    \end{subfigure}
    \hfill
    \begin{subfigure}{\textwidth}
        \centering
        \includegraphics[width=0.8\textwidth]{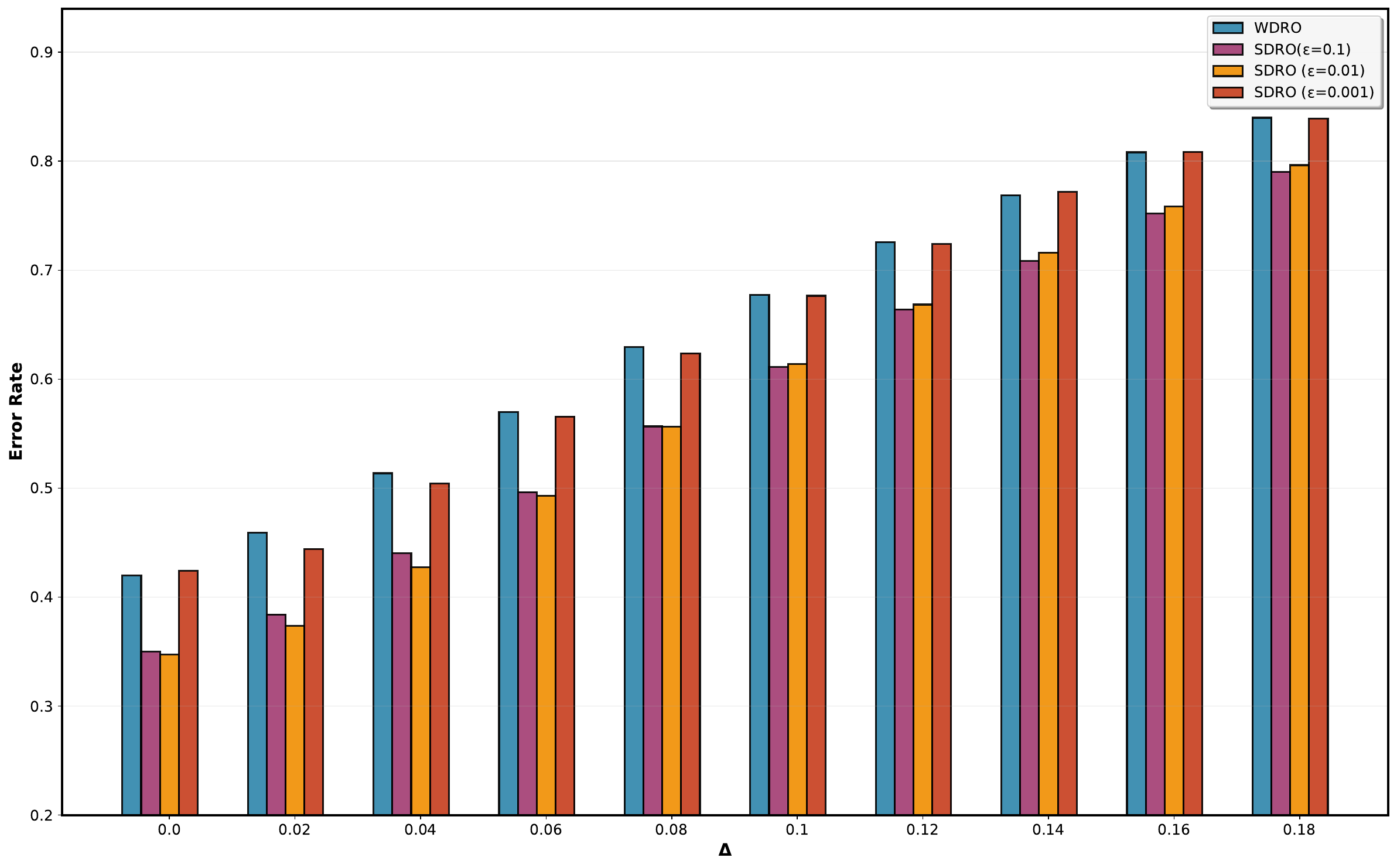}
        \caption{CIFAR-10 Dataset}
        \label{fig:cifar}
    \end{subfigure}
    \caption{
    Results of adversarial training using MNIST or CIFAR-10 datasets for Wasserstein DRO or Sinkhorn DRO models.
    For each plot, the $x$-axis refers to the level of adversarial perturbation of PGD attack, and the $y$-axis refers to the misclassification rate on the perturbed testing dataset.}
    \label{fig:mnist:cifar}
\end{figure}

Figure \ref{fig:mnist:cifar} reports the classification accuracy of the proposed method and the baseline under various adversarial perturbations. On MNIST, our approach consistently outperforms the baseline across all noise levels. When the perturbation budget is small, $\varepsilon=0.1$ yields the highest accuracy; as the budget increases, $\varepsilon=1$ becomes superior. This observation indicates that a larger $\varepsilon$ strengthens robustness against stronger attacks. On CIFAR-10, the same trend holds, $\varepsilon=0.01$ already improves upon the baseline. In our experiments the proposed defense is highly sensitive to the choice of $\varepsilon$. Setting $\varepsilon = 1$  prevents the model from completing the classification task: during training, clean accuracy collapses to $10\%$, indicating that the training set itself is being perturbed to the point where the network can no longer learn any meaningful image features.

\section{Conclusion}\label{Sec:conclusion}
In this work, we developed sampling-based algorithms for solving Sinkhorn DRO, by reformulating it as a bilevel program with multiple infinite-dimensional lower-level subproblems over probability spaces. We proposed a double-loop algorithm (Algorithm~\ref{Alg:solving:SDRO}) with a $\widetilde{\mathcal{O}}(\varrho^{-6})$ complexity guarantee and a more efficient single-loop, mean-field algorithm (Algorithm~\ref{Alg:solving:SDRO:single:improved}) that achieves comparable complexity within an interleaved update framework.
Our analysis extends conventional bilevel optimization techniques to handle infinite-dimensional lower-level subproblems by leveraging tools from sampling theory (such as properties of LSI and SDE-based analysis) to control the error from approximate sampling.

There are several promising directions to explore.
First, our framework could be extended to more general bilevel optimization with infinite-dimensional lower-level variables, such as those arising in meta-learning for generative models (e.g., diffusion or flow-based models).
Second, it is an open question to use iterative sampling-based algorithms to achieve the optimal complexity rate.
Integrating accelerated sampling techniques and developing sharper theoretical analysis could lead to improved guarantees.
%Achieving optimal complexity rates (e.g., $\mathcal{O}(\varrho^{-4})$) remains open. Integrating accelerated sampling techniques and developing tighter non-asymptotic analysis (e.g., via propagation of chaos for finite-particle implementations) could lead to improved guarantees.
%it is promising to integrate faster sampling algorithm and new analysis technique to derive faster algorithm solvers with sharper theoretical guarantees.
Finally, it is interesting to study the Sinkhorn DRO with more general transportation costs and broader classes of loss functions

%We provided a sampling-based algorithm for solving Sinkhorn DRO with quadratic transport cos and general loss function.
%It is promising to integrate faster sampling algorithm from the recent progress of generative AI and to develop better algorithms to solving Sinkhorn DRO with general transport cost and loss function.

\section*{Acknowledgement}
The author would like to thank Jindong Jiang for his help in the numerical study to generate Figure~\ref{fig:mnist:cifar}.
The author is partially sponsored by the Presidential Young Fellowship at The Chinese University of Hong Kong, Shenzhen.

%s
%\clearpage
\bibliographystyle{informs2014} 
\bibliography{references}

\ifnum\paperversion=1
\ECSwitch
\ECHead{Supplementary for \emph{``An Iterative Sampling Approach for Solving Sinkhorn Distributionally Robust Optimization''}}
\fi

\ifnum\paperversion=2
\appendix 
\newpage
\fi

\if\paperversion2
\addcontentsline{toc}{section}{Appendix} %
\appendix
\onecolumn
\fi

\section{Notations and Definitions}\label{Appendix:notation}

We provide a table consisting of related mathematical notations and a list of useful definitions below.
%\begin{spacing}{2}
\begin{table}[htbp]
\centering
\caption{Common Mathematical Notations}
\begin{tabular}{ll}
\toprule
\textbf{Notation} & \textbf{Meaning} \\
\midrule
$\theta$& Decision\\
$z$ & Random vector\\
$d$ & Dimension of random vector $z$\\
$\Reg$ & Entropic regularization parameter of Sinkhorn DRO\\
$\lambda$ & Soft-constrained penalty parameter of Sinkhorn DRO\\
%$(\Reg, \lambda)$ & Hyperparameter (i.e., entropic regularization value and soft constrained penalty parameter) of Sinkhron DRO\\
$x^{(i)}$ & $i$-th collected sample\\
$\mu_*^{\theta}$ & Worst-case distribution for fixed decision $\theta$\\
$\mu_*^{\theta,i}$ & Worst-case distribution for $i$-th observation and fixed decision $\theta$\\
$L_{f,1}, L_{f,2}$ & Lipschitz constants for $f_{\theta}(z)$ and $\nabla f_{\theta}(z)$, respectively\\
$\alpha$ & Log-Sobolev inequality constant for the worst-case distribution\\
$\sigma^2$ & Variance of gradient estimator $\nabla f_{\theta}(z)$ for $z\sim \mu_*^{\theta,i}$ and $i\sim \mathrm{Uniform}([n])$\\
$\delta$ & Accuracy level of Algorithm~\ref{Alg:LSD}\\
$\varrho$ & Error tolerance of bilevel optimization Problem~\eqref{Eq:bilevel:inf}\\
$\mu_k^{(i)}$ & Estimated worst-case distribution for $i$-th observation at $k$-th iteration of Algorithm~\ref{Alg:solving:SDRO:single:improved}\\
$z_k^{(i)}$ & Random sample following distribution $\mu_k^{(i)}$\\
$v_{k+1}$ & Constructed gradient estimator at   $k$-th iteration of Algorithm~\ref{Alg:solving:SDRO:single:improved}\\
$r_{k+1}$ & Constructed momentum gradient estimator at  $k$-th iteration of Algorithm~\ref{Alg:solving:SDRO:single:improved}\\
$\beta_0$ & Momentum parameter\\
$\tau$ & Stepsize parameter for one-step update of Langevin dynamics\\
$\eta$ & Stepsize parameter for Algorithm~\ref{Alg:solving:SDRO} or scaled stepsize parameter for Algorithm~\ref{Alg:solving:SDRO:single:improved}\\
$T_{\text{out}}$ & Number of iterations for Algorithm~\ref{Alg:solving:SDRO}\\
$T$ & Number of iterations for Algorithm~\ref{Alg:LSD} or \ref{Alg:solving:SDRO:single:improved}
\\
$I_k$ & Random sampled set of indices from $[n]$ at $k$-th iteration\\
%$u_k(\cdot)$ & Probability density function of $\mu_k$\\
%$\mathcal{L}(\mu)$ & entropy-regularized functional\\
%$\mathcal{L}_0(\mu)$   & functional that does not contain entropy term\\
%$h(z_k)$ & Wasserstein gradient $\nabla \mathcal{L}_0[\mu_k](z_k)=-\frac{1}{\lambda}\nabla f(z_k) + z_k$ at $k$-th iteration\\
%$\rho$ & Gaussian distribution $\mathcal{N}(0, 2\eta\Reg\mathbf{I}_d)$\\
%$\xi$  & Random variable following distribution $\rho$\\
%$\zeta_k$ & stochastic standard normal Gaussian noise added at $k$-th iteration\\
%$T(\cdot)$ & Transport mapping defined as $T(z) = z - \eta h(z)$.\\
%$J(z)$ & Jacobian matrix $\nabla h(z)$\\
%$\mathrm{ldet}$ & log-determinant operator\\
%$\eta_k$ & Stepsize parameter at $k$-th iteration\\
\bottomrule
\end{tabular}
\end{table}
%\end{spacing}

\paragraph{KL-divergence.}
For two probability distributions $\mu$ and $\nu$, 
define the KL-divergence
\[
\mathcal{D}_{\mathrm{KL}}(\mu\|\nu) = 
\int \log\left( 
\frac{\diff \mu(z)}{\diff \nu(z)}
\right)\diff \mu(z).
\]
Suppose $v(\cdot)$ is the density function of $\nu$, we also write $\mathcal{D}_{\mathrm{KL}}(\mu\|v(\cdot))$ to represent $\mathcal{D}_{\mathrm{KL}}(\mu\|\nu)$ for notational simplicity.

\paragraph{Wasserstein Distance.}
For two probability distributions $\mu$ and $\nu$ and $p\in[1,\infty)$, 
define the $p$-Wasserstein distance
\[
\mathcal{W}_p(\mu,\nu)=
\inf_{\gamma\in\Gamma(\mu,\nu)}~
\Big\{
\left(
\bE_{(x,y)\sim\gamma}[\|x-y\|_2^p]
\right)^{1/p}
\Big\},
\]
where $\Gamma(\mu,\nu)$ denotes the set of joint distributions with marginal distributions being $\mu$ and $\nu$, respectively.
Especially, when $p=1$, the Wasserstein distance has the strong dual reformulation:
\[
\mathcal{W}_1(\mu,\nu)=
\sup_{f:~\|f\|_{\mathrm{Lip}}\le 1}~\Big\{
\bE_{z\sim \mu}[f(z)] - \bE_{z\sim \nu}[f(z)]
\Big\}.
\]

\paragraph{Total variation distance.}
For two probability distributions $\mu$ and $\nu$, 
define the total variation distance
\[
\mathrm{TV}(\mu,\nu) = \sup_{f:~\|f\|_{\infty}\le 1}~\Big\{
\bE_{z\sim \mu}[f(z)] - \bE_{z\sim \nu}[f(z)]
\Big\}.
\]

\paragraph{Fisher divergence.}
For two probability distributions $\mu$ and $\nu$, 
define the Fisher divergence
\[
\mathrm{FD}(\mu\|\nu) = \int \left\| 
\nabla_z\log\left( 
\frac{\mu[z]}{\nu[z]}
\right)
\right\|^2\mu[z]\diff z,
\]
where for a distribution $\mu$, $\mu[z]$ denote its density function at $z$.
%\section{Literature on Sinkhorn DRO}

\paragraph{Entropy.}
For a probability distribution $\mu$, define its (differential) entropy
\[
\mathcal{H}(\mu) = -\int \log(\mu[z])\mu[z]\diff z.
\]

\clearpage
\section{Proofs of Technical Results in Section~\ref{Sec:double:loop:naive}}\label{Appendix:Sec:double:loop:naive}

\begin{proof}[Proof of Proposition~\ref{Proposition:LSI}]
    The first part of this proposition follows from \citep{holley_logarithmic_1987}.
    The second part follows a similar proof idea as in \citep{kim2023symmetric}.
\end{proof}

%\begin{theorem}[Convergence Guarantees of Algorithm~\ref{Alg:LSD}]
%\label{Thm:alg:LSD}
%Assume Assumption~\ref{Assumption:LSD} holds and initialize $\mu_0$ as $\mathcal{N}(0, \Reg \mathbf{I}_d)$.
%Denote by $\mu_t$ the probability distribution of $z_t$ at the $t$-th iteration of Algorithm~\ref{Alg:LSD}.
%If specifying the stepsize $\eta_K\equiv\eta\in(0, \frac{\alpha}{4L^2}]$,
%\[
%D_{\mathrm{KL}}(\mu_T\|\mu_*)\le 
%e^{-\alpha\eta T}D_{\mathrm{KL}}(\mu_0\|\mu_*) +  \frac{8\eta dL^2}{\alpha}.
%\]
%If specifying the stepsize $\eta_k=\frac{8}{3\alpha(t+\tau)}$ with $\tau=\frac{32L^2}{3\alpha^2}$,
%\[
%D_{\mathrm{KL}}(\mu_T\|\mu_*)\le
%\frac{
%2\max\{
%192L^2d,
%9\alpha^2\tau D_{\mathrm{KL}}(\mu_0\|\mu_*)
%\}
%}{9\alpha^2 (T+\tau)}.
%\]
%% and specify the stepsize $\eta_k\in(0, \frac{\alpha}{4L_f^2}]$.
%% Then, it holds that 
%% \[
%% D_{\mathrm{KL}}(\mu_T\|\mu_*)\le D_{\mathrm{KL}}(\mu_0\|\mu_*)\cdot e^{-\alpha\Reg T} + \frac{8\Reg d L_f^2}{\alpha}.
%% \]
%\end{theorem}

\begin{proof}[Proof of Theorem~\ref{Theorem:alg:LSD}]
It can be shown that $V_i(\theta,z):=-\frac{1}{\lambda\Reg}f_{\theta}(z) + \frac{1}{2\Reg}\|z - x^{(i)}\|_2^2$ is $L_{V,2}$-smooth with
\[
L_{V,2}:=\frac{1}{\Reg}\left( 1 + \frac{L_{f,2}}{\lambda}
\right).
%\frac{1}{\lambda\Reg}L_{f,2} + \frac{1}{\Reg}.
\]
Also, $\mu_{*}^{i,\theta}[z]:=\frac{\diff\mu_{*}^{i,\theta}}{\diff z}(z)\propto \exp(-V_i(\theta,z))$ satisfies LSI with constant $\alpha$.
Based on \citep[Lemma~3]{vempala2019rapid}, it holds that 
\begin{align*}
D_{\mathrm{KL}}(\mu_T\|\mu_*^{\theta,i})&\le 
e^{-\alpha\tau\Reg T}D_{\mathrm{KL}}(\mu_0\|\mu_*^{\theta,i})
+
\frac{8\tau\Reg dL_{V,2}^2}{\alpha}.
\end{align*}
Since $\tau\le \frac{\alpha\Reg}{4(1 + L_{f,2}/\lambda)^2}\cdot \frac{\delta^2\alpha}{8d}$, we derive
\[
\frac{8\tau\Reg dL_{V,2}^2}{\alpha}\le \frac{\alpha\delta^2}{4}.
\]
Since $T\ge \frac{1}{\alpha\tau\Reg}\log\frac{4D_{\mathrm{KL}}(\mu_0\|\mu_*^{\theta,i})}{\delta^2\alpha}$, we derive
\[
e^{-\alpha\tau\Reg T}D_{\mathrm{KL}}(\mu_0\|\mu_*^{\theta,i})\le \frac{\alpha\delta^2}{4}.
\]
Combining those terms together, we obtain that
\[
D_{\mathrm{KL}}(\mu_T\|\mu_*^{\theta,i})\le \frac{\alpha\delta^2}{2}.
\]
As $\mu_*^{\theta,i}$ satisfies Talagrand's inequality with constant $\alpha$, we derive that
\[
\mathcal{W}_2(\mu_T, \mu_*^{\theta,i})\le \sqrt{
\frac{2D_{\mathrm{KL}}(\mu_T\|\mu_*^{\theta,i})}{\alpha}
}\le \delta.
\]

% We first provide the analysis of Algorithm~\ref{Alg:LSD} within one step of update with stepsize $\eta$.
% Define
% \[
% \mathcal{L}(\mu) = D_{\mathrm{KL}}(\mu\|\mu_*)=
% \bE_{\mu}[\frac{-f(z)}{\lambda}] + \frac{1}{2}\bE_{\mu}[\|z\|_2^2] - \Reg\mathcal{H}(\mu) + \text{Constant}.
% \]
% and $L=L_f/\lambda + 1$.
% For the update 
% \[
% z_1 \leftarrow z_0 - \eta h(z_0) + \sqrt{2\eta\Reg}\zeta_0,
% \]
% it can be shown by \citep{vempala2019rapid} that for $\eta\in(0, \frac{\alpha}{4L^2}]$,
% \[
% \mathcal{L}(\mu_1)\le e^{-\alpha\eta}\mathcal{L}(\mu_0) + 6\eta^2dL^2.
% \]
% \begin{itemize}
%     \item 
% If using the constant stepsize $\eta_k\equiv \eta$ in Algorithm~\ref{Alg:LSD}, we obtain that
% \[
% \mathcal{L}(\mu_T)\le e^{-\alpha\eta T}\mathcal{L}(\mu_0) +  
% \frac{6\eta^2dL^2}{1 - e^{-\alpha\eta}}\le 
% e^{-\alpha\eta T}\mathcal{L}(\mu_0) +  \frac{8\eta dL^2}{\alpha}.
% \]
%     \item
% If using the decaying stepsize $\eta_k=\frac{8}{3\alpha(t+\tau)}$ with $\tau=\frac{32L^2}{3\alpha^2}$ in Algorithm~\ref{Alg:LSD}, by induction, we obtain that
% \[
% \mathcal{L}(\mu_T)\le
% \frac{
% 32\max\{
% 12L^2d,
% \frac{9}{16}\alpha^2\tau\mathcal{L}(\mu_0)
% \}
% }{9\alpha^2 (T+\tau)}.
% \]
% \end{itemize}
% % When $\mu_0\sim \mathcal{N}(0,\Reg\mathbf{I}_d)$, we find
% % \[
% % \mathcal{L}(\mu_0)
% % =
% % \bE_{\mu}[\frac{-f(z)}{\lambda}] + \frac{d\Reg}{2} - \frac{d\Reg}{2}\log(2\pi e\Reg)
% % \le \frac{-f(0)}{\lambda} + 
% % \]
The proof is completed.
\end{proof}

\begin{proof}[Proof of Lemma~\ref{Lemma:hyperparameter}]
\begin{itemize}
    \item 
For the bias part, it holds that 
\begin{align*}
&\left\|\bE\big[\widehat{\nabla}F(\theta; z)\big]
-
{\nabla}F(\theta)
\right\|_2\\
\le&\bE_{i\sim\mathrm{Uniform}([n])}\left\|\bE_{\tilde{\mu}^{i,\theta}}\big[\nabla_{\theta}f_{\theta}(z)\big]
-
\bE_{{\mu}^{i,\theta}_*}\big[\nabla_{\theta}f_{\theta}(z)\big]
\right\|_2\\
\le&\bE_{i\sim\mathrm{Uniform}([n])}\big[ 
\mathcal{W}_1(\tilde{\mu}^{i,\theta}, {\mu}^{i,\theta}_*)\cdot L_{f,2}
\big]\\
\le&\bE_{i\sim\mathrm{Uniform}([n])}\big[ 
\mathcal{W}_2(\tilde{\mu}^{i,\theta}, {\mu}^{i,\theta}_*)\cdot L_{f,2}
\big]\le \delta\cdot L_{f,2},
\end{align*}
where in the first inequality, we denote by $\tilde{\mu}^{i,\theta}$ the probability distribution such that $\mathcal{W}_2(\tilde{\mu}^{i,\theta}, {\mu}^{i,\theta}_*)\le \delta$.
    \item
For the variance part, it holds that
\[
\mathbb{V}\mathrm{ar}(\widehat{\nabla}F(\theta; z))=
\frac{1}{n}\sum_{i=1}^nA_i,
\]
where
\[
A_i = \bE_{z\sim\tilde{\mu}^{\theta, i}}
\left\| 
\nabla_{\theta}f_{\theta}(z) 
-
\frac{1}{n}\sum_{i=1}^n\bE_{z\sim\tilde{\mu}^{\theta, i}}[\nabla_{\theta}f_{\theta}(z) ]
\right\|_2^2.
\]
Based on the triangular inequality, 
\begin{equation}\label{Eq:tri:Ai}
\begin{multlined}
A_i\le \bE_{z\sim{\mu}_*^{\theta, i}}
\left\| 
\nabla_{\theta}f_{\theta}(z) 
-
\frac{1}{n}\sum_{i=1}^n\bE_{z\sim\tilde{\mu}^{\theta, i}}[\nabla_{\theta}f_{\theta}(z) ]
\right\|_2^2\\
+\left| 
\bE_{z\sim{\mu}_*^{\theta, i} - \tilde{\mu}^{\theta, i}}
\left\| 
\nabla_{\theta}f_{\theta}(z) 
-
\frac{1}{n}\sum_{i=1}^n\bE_{z\sim\tilde{\mu}^{\theta, i}}[\nabla_{\theta}f_{\theta}(z) ]
\right\|_2^2
\right|
\end{multlined}
\end{equation}
Note that in the expression above, we write $\mathbb{E}_{\mu-\nu}[f(z)]$ to denote the difference between $\mathbb{E}_{\mu}[f(z)]$ and $\mathbb{E}_{\nu}[f(z)]$ for generic probability distributions $\mu,\nu$ and measurable function $f(z)$.
For the first component on the right-hand side of \eqref{Eq:tri:Ai}, it can be further bounded as 
\begin{align*}
&2\bE_{z\sim{\mu}_*^{\theta, i}}
\left\| 
\nabla_{\theta}f_{\theta}(z) 
-
\frac{1}{n}\sum_{i=1}^n\bE_{z\sim{\mu}_*^{\theta, i}}[\nabla_{\theta}f_{\theta}(z) ]
\right\|_2^2
+2
\left\| 
\frac{1}{n}\sum_{i=1}^n\bE_{z\sim\tilde{\mu}^{\theta, i} - {\mu}_*^{\theta, i}}[\nabla_{\theta}f_{\theta}(z) ]
\right\|_2^2\\
\le&2\bE_{z\sim{\mu}_*^{\theta, i}}
\left\| 
\nabla_{\theta}f_{\theta}(z) 
-
\frac{1}{n}\sum_{i=1}^n\bE_{z\sim{\mu}_*^{\theta, i}}[\nabla_{\theta}f_{\theta}(z) ]
\right\|_2^2 + 2L_{f,2}^2\delta^2.
\end{align*}
For the second component on the right-hand side of \eqref{Eq:tri:Ai}, from the proof of Theorem~\ref{Theorem:alg:LSD}, it can be shown that 
\[
\mathrm{TV}({\mu}_*^{\theta, i}, \tilde{\mu}^{\theta, i})\le
\sqrt{\frac{1}{2}\mathcal{D}_{\mathrm{KL}}(\tilde{\mu}^{\theta, i}\|{\mu}_*^{\theta, i})}\le 
\frac{\delta\sqrt{\alpha}}{2},
\]
and
\[
\left\| 
\nabla_{\theta}f_{\theta}(z) 
-
\frac{1}{n}\sum_{i=1}^n\bE_{z\sim\tilde{\mu}^{\theta, i}}[\nabla_{\theta}f_{\theta}(z) ]
\right\|_2^2\le 4L_{f,1}^2,
\]
and therefore
\[
\left| 
\bE_{z\sim{\mu}_*^{\theta, i} - \tilde{\mu}^{\theta, i}}
\left\| 
\nabla_{\theta}f_{\theta}(z) 
-
\frac{1}{n}\sum_{i=1}^n\bE_{z\sim\tilde{\mu}^{\theta, i}}[\nabla_{\theta}f_{\theta}(z) ]
\right\|_2^2
\right|
\le 2\sqrt{\alpha}\delta L_{f,1}^2.
\]
Combining these bounds togehter, we arrive that 
\begin{align*}
&\mathbb{V}\mathrm{ar}(\widehat{\nabla}F(\theta; z))\\
\le&
\frac{1}{n}\sum_{i=1}^n\left( 
2\bE_{z\sim{\mu}_*^{\theta, i}}
\left\| 
\nabla_{\theta}f_{\theta}(z) 
-
\frac{1}{n}\sum_{i=1}^n\bE_{z\sim{\mu}_*^{\theta, i}}[\nabla_{\theta}f_{\theta}(z) ]
\right\|_2^2 + 2L_{f,2}^2\delta^2
\right) + 2\sqrt{\alpha}\delta L_{f,1}^2\\
\le&2\sigma^2 + 2L_{f,2}^2\delta^2 + 2\sqrt{\alpha}\delta L_{f,1}^2.
\end{align*}
\item 
The last part of this lemma follows from Theorem~\ref{Theorem:alg:LSD}.
\end{itemize}

\end{proof}

\begin{proof}[Proof of Theorem~\ref{Theorem:complexity:bound}]
Denote by $\widehat{F}(\theta)$ the objective corresponding to the gradient $\bE\widehat{\nabla} F(\theta)$.
We have the following error decomposition:
\begin{align*}
&\bE\|\nabla F(\widehat{\theta})\|_2^2\le 2\bE\|\nabla \widehat{F}(\widehat{\theta})\|_2^2 + 2\bE\|\nabla \widehat{F}(\widehat{\theta}) - \nabla F(\widehat{\theta})\|_2^2\\
\le&2\bE\|\nabla \widehat{F}(\widehat{\theta})\|_2^2 + 2L_{f,2}^2\delta^2.
\end{align*}
As $\widehat{F}(\widehat{\theta})$ is $L_{f,2}$-smooth in $\theta$, according to the convergence analysis in~\citep{bottou2018optimization}, the output in Algorithm~\ref{Alg:solving:SDRO} with constant stepsize $\eta$ satisfies
\[
\bE\|\nabla \widehat{F}(\widehat{\theta})\|_2^2
\le 
\frac{2\big( 
\widehat{F}(\theta_0)-\min_{\theta}\widehat{F}(\theta)
\big)}{\eta T_{\text{out}}} + L_{f,2}\eta\mathbf{V}.
\]
We specify the stepsize $\eta=1/\sqrt{T_{\text{out}}\mathbf{V}}$, then
\[
\bE\|\nabla \widehat{F}(\widehat{\theta})\|_2^2\le 
\frac{\sqrt{\mathbf{V}}
\big( 
2\big( 
\widehat{F}(\theta_0)-\min_{\theta}\widehat{F}(\theta)
\big)
+
L_{f,2}
\big)
}{\sqrt{T_{\text{out}}}}
\]
In order to ensure $\bE\|\nabla F(\widehat{\theta})\|_2^2\le \varrho^2$, we take
\[
2L_{f,2}^2\delta^2\le \frac{1}{2}\varrho^2,\quad 
\frac{\sqrt{\mathbf{V}}
\big( 
2\big( 
\widehat{F}(\theta_0)-\min_{\theta}\widehat{F}(\theta)
\big)
+
L_{f,2}
\big)
}{\sqrt{T_{\text{out}}}}\le \frac{1}{4}\varrho^2.
\]
Thus, we take $\delta = \frac{\varrho}{2L_{f,2}}$ and 
\[
T_{\text{out}}\ge 16\mathbf{V}
\big( 
2\big( 
\widehat{F}(\theta_0)-\min_{\theta}\widehat{F}(\theta)
\big)
+
L_{f,2}
\big)^2\varrho^{-4}.
\]
The proof is completed.
\end{proof}

\clearpage
\section{Proofs of Technical Results in Section~\ref{Sec:proof:cov:ana}}
\label{Appendix:Sec:proof:cov:ana}

\begin{proof}[Proof of Lemma~\ref{Lemma:descent}]
It is easy to verify that $F(\theta)$ is $L_{f,2}$-smooth in $\theta$.
It follows that
\begin{align*}
F(\theta_{k+1})&\le F(\theta_k) + \nabla F(\theta_k)^\top(\theta_{k+1} - \theta_k) + \frac{L_{f,2}}{2}\|\theta_{k+1} - \theta_k\|_2^2\\
&=F(\theta_k) - \eta\tau \nabla F(\theta_k)^\top r_{k+1} + \frac{L_{f,2}\tau^2\eta^2}{2}\|r_{k+1}\|_2^2\\
&=F(\theta_k) + \frac{\eta\tau}{2}\|\nabla F(\theta_k) - r_{k+1}\|_2^2 - \frac{\tau\eta}{2}\|\nabla F(\theta_k)\|_2^2 + \left( 
\frac{L_{f,2}\eta^2\tau^2}{2} - \frac{\eta\tau}{2}
\right)\|r_{k+1}\|_2^2,
\end{align*}
where the first equality is by the relation that $\theta_{k+1} = \theta_k - \eta\tau r_{k+1}$.

Since $\eta\tau\le \frac{1}{2L_{f,2}}$, it holds that $\frac{L_{f,2}\eta^2\tau^2}{2} - \frac{\eta\tau}{2}\le -\frac{\eta\tau}{4}$.
Therefore, the desired result holds.
\end{proof}

\begin{proof}[Proof of Lemma~\ref{Lemma:gradient:diff}]
By the relation $r_{k+1} = (1-\beta_0)r_k + \beta_0v_{k+1}$, it follows that
\begin{equation}\label{Eq:Lemma:4:3:1}
\begin{aligned}
&\bE\big[\|\nabla F(\theta_k) - r_{k+1}\|_2^2\big|\mathcal{G}_{k-1}\big]\\
=&\bE\big[\|\nabla F(\theta_k) - (1-\beta_0)r_k - \beta_0v_{k+1}\|_2^2\big|\mathcal{G}_{k-1}\big]\\
=&\bE\big[\big\|
(1-\beta_0)(\nabla F(\theta_{k-1}) - r_k)
+
(1-\beta_0)(\nabla F(\theta_k) - \nabla F(\theta_{k-1}))
\\
&\qquad+
\beta_0(\nabla F(\theta_k) - \nabla F(\theta_k; \mu_{k}^{(1:n)}))
+
\beta_0(\nabla F(\theta_k; \mu_{k}^{(1:n)}) - v_{k+1})
\big\|_2^2\big|\mathcal{G}_{k-1}\big]\\
=&\big\|
(1-\beta_0)(\nabla F(\theta_{k-1}) - r_k)
+
(1-\beta_0)(\nabla F(\theta_k) - \nabla F(\theta_{k-1}))\\
&\qquad+
\beta_0(\nabla F(\theta_k) - \nabla F(\theta_k; \mu_{k}^{(1:n)}))\big\|_2^2
+\beta_0^2\bE\big[\big\|\nabla F(\theta_k; \mu_{k}^{(1:n)}) - v_{k+1}
\big\|_2^2\big|\mathcal{G}_{k-1}\big],
\end{aligned}
\end{equation}
where the last equality is because, conditioned on $\mathcal{G}_{k-1}$, the term
\[
\mathbf{A}_1:=
(1-\beta_0)(\nabla F(\theta_{k-1}) - r_k)
+
(1-\beta_0)(\nabla F(\theta_k) - \nabla F(\theta_{k-1}))
+
\beta_0(\nabla F(\theta_k) - \nabla F(\theta_k; \mu_{k}^{(1:n)}))
\]
is deterministic, and $\mathbb{E}[\nabla F(\theta_k; \mu_{k}^{(1:n)}) - v_{k+1}\mid \mathcal{G}_{k-1}]=0$.

Furthermore, due to the relation $\|a+b\|_2^2\le (1+\beta)\|a\|_2^2 + (1 + \frac{1}{\beta})\|b\|_2^2$ for any $\beta>0$, it holds that 
\begin{equation}\label{Eq:Lemma:4.3:2}
\begin{aligned}
\|\mathbf{A}_1\|_2^2&\le (1+\beta_0)(1-\beta_0)^2\|\nabla F(\theta_{k-1}) - r_k\|_2^2 \\
&\qquad+ (1 + \frac{1}{\beta_0})\|(1-\beta_0)(\nabla F(\theta_k) - \nabla F(\theta_{k-1}))
+
\beta_0(\nabla F(\theta_k) - \nabla F(\theta_k; \mu_{k}^{(1:n)}))\|_2^2\\
&\le (1+\beta_0)(1-\beta_0)^2\|\nabla F(\theta_{k-1}) - r_k\|_2^2\\
&\quad + 2(1 + \frac{1}{\beta_0})\Big[ 
(1-\beta_0)^2\|\nabla F(\theta_k) - \nabla F(\theta_{k-1})\|_2^2 + \beta_0^2\|\nabla F(\theta_k) - \nabla F(\theta_k; \mu_{k}^{(1:n)})\|_2^2
\Big]\\
&\le(1-\beta_0)\|\nabla F(\theta_{k-1}) - r_k\|_2^2+ \frac{4}{\beta_0}
\|\nabla F(\theta_k) - \nabla F(\theta_{k-1})\|_2^2\\
&\quad+ 4\beta_0\|\nabla F(\theta_k) - \nabla F(\theta_k; \mu_{k}^{(1:n)})\|_2^2
\end{aligned}
\end{equation}
where the second inequality uses the relation $\|a+b\|_2^2\le 2\|a\|_2^2 + 2\|b\|_2^2$, and the last inequality uses the assumption that $\beta_0\in(0,1]$.

By the Lipschitz smoothness assumption of $F(\theta)$, 
\begin{equation}\label{Eq:Lemma:4.3:3}
    \|\nabla F(\theta_k) - \nabla F(\theta_{k-1})\|_2^2\le L_{f,2}^2\|\theta_k - \theta_{k-1}\|_2^2=L_{f,1}^2\tau^2\eta^2\|r_k\|_2^2.
\end{equation}
Combining relations \eqref{Eq:Lemma:4:3:1}, \eqref{Eq:Lemma:4.3:2}, and \eqref{Eq:Lemma:4.3:3}, we obtain the desired result.
\end{proof}

% We first show that
% \begin{multline*}
% \bE\big[\|\nabla F(\theta_k) - r_{k+1}\|_2^2\big|\mathcal{G}_{k-1}\big]
% \le
% (1-\beta_0)\|\nabla F(\theta_{k-1}) - r_{k}\|_2^2 
% + 
% \frac{4L_{f,2}^2}{\beta_0}\|\theta_k - \theta_{k-1}\|_2^2 \\+ 
% 4\beta_0\left\| \nabla F(\theta_k) - \nabla F(\theta_k; \mu_{k}^{(1:n)})\right\|_2^2 
% + \beta_0^2\bE\left[\left\|\nabla F(\theta_k; \mu_{k}^{(1:n)}) - v_{k+1}\right\|_2^2\middle|\mathcal{G}_{k-1}\right].
% \end{multline*}
% Besides, it can be shown that 
% \begin{align*}
% &\left\| \nabla F(\theta_k) - \nabla F(\theta_k; \mu_{k}^{(1:n)})\right\|_2
% \le\frac{1}{n}\sum_{i\in[n]}\left\|
% \bE_{\mu_{k}^{(i)} - \mu_*^{i, \theta_k}}[\nabla f_{\theta_t}(z)]
% \right\|_2\\
% \le&\frac{1}{n}\sum_{i\in[n]}L_{f,1}\cdot\mathrm{TV}(\mu_{k}^{(i)}, \mu_*^{i, \theta_k})\\
% \le&\frac{L_{f,1}}{n}\sum_{i\in[n]}\sqrt{
% \mathcal{D}_{\mathrm{KL}}(\mu_{k}^{(i)}, \mu_*^{i, \theta_k})
% },
% \end{align*}
% which further implies
% \[
% \left\| \nabla F(\theta_k) - \nabla F(\theta_k; \mu_{k}^{(1:n)})\right\|_2^2\le 
% \frac{L_{f,1}^2}{n}\sum_{i\in[n]}\mathcal{D}_{\mathrm{KL}}(\mu_{k}^{(i)}, \mu_*^{i, \theta_k}).
% \]
% Also,
% \[
% \bE\left[\left\|\nabla F(\theta_k; \mu_{k}^{(1:n)}) - v_{k+1}\right\|_2^2\middle|\mathcal{G}_{k-1}\right]
% \le \frac{L_{f,2}^2}{|I_k|}.
% \]
\begin{proof}[Proof of Lemma~\ref{Lemma:KL:div}]
Throughout the proof, we use the stepsize parameters
\begin{equation}
\tau \le \min\left(1, 
\frac{1}{L_{G,2}}, \sqrt{\frac{\lambda}{\Reg L_{f,2}^2}}, \frac{1}{\alpha}, \frac{\alpha}{4L_{G,2}^2}
\right),\qquad \eta\le 1.
\end{equation}
Let us define the auxiliary distribution $\tilde{\mu}_{k+1}^{(i)}$ as the law of $\tilde{z}_{k+1}^{(i)}$, where
\[
\tilde{z}_{k+1}^{(i)} = z_k^{(i)} - \tau\left( 
    -\frac{\nabla_zf_{\theta_k}(z_k^{(i)})}{\lambda} + (z_k^{(i)} - x^{(i)})
    \right) + \sqrt{2\tau\Reg}\zeta_k^{(i)},\quad 
    \mathrm{Law}(z_k^{(i)})=\mu_{k}^{(i)}.
\]
Therefore, it holds that 
\begin{equation}\label{Eq:KL:mu:mu:sta}
\begin{aligned}
&\bE\big[ 
\mathcal{D}_{\mathrm{KL}}(\mu_{k+1}^{(i)} \| \mu_*^{i, \theta_{k+1}})\big| \mathcal{G}_{k}
\big]\\=&
\left( 
1 - \frac{|I_k|}{n}
\right)\bE\big[ \mathcal{D}_{\mathrm{KL}}(\mu_{k}^{(i)} \| \mu_*^{i, \theta_{k+1}})\big| \mathcal{G}_{k}
\big]
+
\frac{|I_k|}{n}\bE\big[ \mathcal{D}_{\mathrm{KL}}(\tilde{\mu}_{k+1}^{(i)} \| \mu_*^{i, \theta_{k+1}})\big| \mathcal{G}_{k}
\big],
\end{aligned}
\end{equation}
where $\bE[\cdot\mid\mathcal{G}_k]$ refers to the expected value with respect to the randomness conditioned on $\mathcal{G}_k$, i.e., taking the expected value over the randomness of the sampling indices $I_k$.
We upper bound each of the component on the right-hand-side in the following.
\paragraph{Bounding $\bE\big[ \mathcal{D}_{\mathrm{KL}}(\mu_{k}^{(i)} \| \mu_*^{i, \theta_{k+1}})\big| \mathcal{G}_{k}
\big]$.}
Define $\theta_{k+1}$ as the output at time $\tau$ of the dynamics
\begin{equation}
\vartheta_t = \theta_k - t\eta r_{k+1},\quad 
t\in[0,\tau].
\label{Eq:expression:vartheta}
\end{equation}
For a given probability measure $\mu$, we write $\mu[z]$ for the density function $\frac{\diff \mu(z)}{\diff z}$.
For fixed $I_k$, by the chain rule, it holds that 
\begin{equation}
\label{Eq:diff:t:div}
\begin{aligned}
&\frac{\diff}{\diff t}\mathcal{D}_{\mathrm{KL}}(\mu_{k}^{(i)} \| \mu_*^{i, \vartheta_{t}})=-
\int \frac{\mu_{k}^{(i)}[z]}{\mu_*^{i, \vartheta}[z]}\frac{\diff}{\diff t}\left( 
\mu_*^{i, \vartheta_{t}}[z]
\right)\diff z.
% \\
% =&\eta\int \mu_{k}^{(i)}\left[ 
% \frac{1}{\lambda\Reg}\inp{\nabla f_{\vartheta_{t}}(z)}{r_{k+1}}
% -
% \frac{1}{\lambda\Reg}\bE_{\mu_*^{i, \vartheta_{t}}}[\inp{\nabla f_{\vartheta_{t}}(z)}{r_{k+1}}]
% \right]\\
% \le&
% \frac{\eta(L_{f,1}^2 + \|r_{k+1}\|_2^2)}{\lambda\Reg}
% .
\end{aligned}
\end{equation}
Define the function $V_i(\theta,z):=-\frac{1}{\lambda\Reg}f_{\theta}(z) + \frac{1}{2\Reg}\|z - x^{(i)}\|_2^2$ and scalar $Z_{i,\theta} := \int V_i(\theta,z)\diff z$.
Then the density $\mu_*^{i, \vartheta_{t}}[z] = \exp(-V_i(\vartheta_t,z)) / Z_{i,\vartheta_{t}}$. By the chain rule, it holds that
\begin{equation}
\begin{aligned}
&\frac{\diff}{\diff t}\left( 
\mu_*^{i, \vartheta_{t}}[z]
\right)=\left\langle 
\frac{\diff}{\diff\vartheta_{t}}\left(
\mu_*^{i, \vartheta_{t}}[z]
\right),~
\frac{\diff \vartheta_{t}}{\diff t}
\right\rangle\\
=&
\left\langle 
-\nabla_1V_i(\vartheta_t,z)\cdot \mu_*^{i, \vartheta_{t}}[z] + 
\bE_{z\sim \mu_*^{i, \vartheta_{t}}}\big[ 
\nabla_1V_i(\vartheta_t,z)
\big]\cdot \mu_*^{i, \vartheta_{t}}[z],~
-\eta r_{k+1}
\right\rangle
\\
=&\frac{\eta}{\lambda\Reg}\mu_*^{i, \vartheta_{t}}[z]\cdot \langle 
\nabla_{\theta}f_{\vartheta_{t}}(z) + \bE_{z\sim \mu_*^{i, \vartheta_{t}}}\big[ 
f_{\vartheta_{t}}(z)\big],~~r_{k+1}
\rangle.
\end{aligned}\label{Eq:diff:t:mu}
\end{equation}
Substituting \eqref{Eq:diff:t:mu} into \eqref{Eq:diff:t:div}, we obtain that 
\begin{align*}
&\frac{\diff}{\diff t}\mathcal{D}_{\mathrm{KL}}(\mu_{k}^{(i)} \| \mu_*^{i, \vartheta_{t}})
=
-\frac{\eta}{\lambda\Reg}\int \langle 
\nabla_{\theta}f_{\vartheta_{t}}(z) + \bE_{z\sim \mu_*^{i, \vartheta_{t}}}\big[ 
f_{\vartheta_{t}}(z)\big],~~r_{k+1}
\rangle\cdot \mu_k^{(i)}[z]\diff z\\
\le&\frac{\eta}{\lambda\Reg}(L_{f,1}^2 + \|r_{k+1}\|_2^2).
\end{align*}

By the fundamental theorem of calculus, it holds that
\[
\mathcal{D}_{\mathrm{KL}}(\mu_{k}^{(i)} \| \mu_*^{i, \vartheta_{t}})\le 
\mathcal{D}_{\mathrm{KL}}(\mu_{k}^{(i)} \| \mu_*^{i, \theta_k}) + \frac{\eta t(L_{f,1}^2 + \|r_{k+1}\|_2^2)}{\lambda\Reg},\quad \forall t\in[0,\tau].
\]
Therefore, we obtain that 
\begin{equation}
\bE\big[ \mathcal{D}_{\mathrm{KL}}(\mu_{k}^{(i)} \| \mu_*^{i, \theta_{k+1}})\big| \mathcal{G}_{k}
\big]
\le 
\mathcal{D}_{\mathrm{KL}}(\mu_{k}^{(i)} \| \mu_*^{i, \theta_k}) + 
\frac{\eta\tau L_{f,1}^2}{\lambda\Reg} + \frac{\eta\tau}{\lambda\Reg}\bE\big[ \|r_{k+1}\|_2^2\big| \mathcal{G}_{k}
\big].\label{Eq:upper:lemma:1}
\end{equation}

\paragraph{Bounding $\bE\big[ \mathcal{D}_{\mathrm{KL}}(\tilde{\mu}_{k+1}^{(i)} \| \mu_*^{i, \theta_{k+1}})\big| \mathcal{G}_{k}
\big]$.}
% Next, we try to bound
% \[
% \mathcal{D}_{\mathrm{KL}}(\tilde{\mu}_{k+1}^{(i)} \| \mu_*^{i, \theta_{k+1}})
% \]
Define 
\begin{equation}
G_i(\theta,z) = -\frac{f_{\theta}(z)}{\lambda} + \frac{1}{2}\|z - x^{(i)}\|_2^2.\label{Eq:V:z:theta}
\end{equation}
It can be shown that $G_i(z,\theta)$ is $L_{G,2}$-smooth in $z$ with 
\begin{equation}
L_{G,2}:=1 + \frac{1}{\lambda}L_{f,2}.
\end{equation}
%Also, $V(z,\theta)$ is $\frac{L_{f,2}}{\lambda}$-smooth in $\theta$.
Let us define $\rho_0:={\mu}_{k}^{(i)}$ with $z_0\sim\rho_0$, and
\[
\diff z_t = -\nabla_2G_i(\theta_k,z_0)\diff t + \sqrt{2\Reg}\diff\mathbf{B}_t,\quad 
\mathrm{Law}(z_t)=\rho_t.
\]
Then, $\tilde{\mu}_{k+1}^{(i)}$ has the same distribution of the output of the SDE above at time $\tau$.
Recall we define $\theta_{k+1}$ as the output at time $\tau$ of \eqref{Eq:expression:vartheta}.
%Also, $\theta_{k+1}$ is the output at time $\tau$ of
% \[
% \vartheta_t = \theta_k - t\eta r_{k+1}.
% \]
Following the same calculation procedure as in \citep[Appendix~B]{marion2024implicit}, we find the time derivative
\begin{equation}\label{Eq:time:div:KL}
\begin{aligned}
&\frac{\diff}{\diff t}\mathcal{D}_{\mathrm{KL}}(\rho_t \| \mu_*^{i, \vartheta_t})
=-\mathrm{FD}(\rho_t\|\mu_*^{i, \vartheta_t})\\
&\qquad + 
\int \left\langle
\nabla\log\left( 
\frac{\rho_t[z_t]}{\mu_*^{i, \vartheta_t}[z_t]}
\right),~~
\nabla_2G_i(\theta_k, z_t)
-
\bE_{z_0\sim\rho_{0\mid t}}[\nabla_2G_i(\theta_k,z_0)\mid z_t]
\right\rangle \rho_t[z_t]\diff z_t\\
&\qquad + 
\int\left\langle
\nabla\log\left( 
\frac{\rho_t[z_t]}{\mu_*^{i, \vartheta_t}[z_t]}
\right),~~
\nabla_2G_i(\vartheta_t, z_t) - 
\nabla_2G_i(\theta_k, z_t)
\right\rangle\rho_t[z_t] \diff z_t\\
&\qquad + \frac{\eta}{\lambda\Reg}
\int 
\left[ 
\inp{\nabla f_{\vartheta_{t}}(z_t)}{r_{k+1}}
-
\bE_{\mu_*^{i, \vartheta_{t}}}[\inp{\nabla f_{\vartheta_{t}}(z_t)}{r_{k+1}}]
\right]\rho_t[z_t]\diff z_t.%\\
%:=&a_1+a_2+a_3+b.
\end{aligned}    
\end{equation}
Denote by the second, third and fourth components of the right-hand-side above as $\mathbf{A}_2, \mathbf{A}_3, \mathbf{A}_4$, respectively.
We provide upper bound on those expressions below.
\begin{itemize}
    \item By the law of total expectation, it holds that
\begin{equation*}
\begin{aligned}
\mathbf{A}_2=&\int \left\langle
\nabla\log\left( 
\frac{\rho_t[z_t]}{\mu_*^{i, \vartheta_t}[z_t]}
\right),~~
\nabla_2G_i(\theta_k, z_t)
-
\nabla_2G_i(\theta_k,z_0)
\right\rangle \rho_{0,t}[z_0, z_t]\diff(z_0,z_t)\\
\le&\int 
\Bigg(
\left\| 
\nabla_2G_i(\theta_k, z_t)
-
\nabla_2G_i(\theta_k,z_0)
\right\|_2^2 \\
&\qquad\qquad\qquad\qquad\qquad+
\frac{1}{4}\left\|
\nabla\log\left( 
\frac{\rho_t[z_t]}{\mu_*^{i, \vartheta_t}[z_t]}
\right)
\right\|_2^2
\Bigg)
 \rho_{0,t}[z_0, z_t]
\diff(z_0,z_t)\\
\le& L_{G,2}^2
\bE_{(z_0,z_t)\sim\rho_{0,t}}\|z_t - z_0\|_2^2 + 
\frac{1}{4}\mathrm{FD}(\rho_t\|\mu_*^{i, \vartheta_t}),
\end{aligned}
\end{equation*}
where the first inequality is based on the relation $\inp{a}{b}\le \|a\|_2^2 + \frac{1}{4}\|b\|_2^2$, and the second inequality is because $G_i(\theta,z)$ is $L_{G,2}$-smooth in $z$.
Based on Lemma~\ref{Lemma:norm:2:diff}, it further holds that
\begin{equation}\label{Eq:expr:A2}
\mathbf{A}_2\le 
\frac{4t^2L_{G,2}^4}{\alpha}\mathcal{D}_{\mathrm{KL}}(\rho_0\| \mu_*^{i, \theta_k}) + 2t^2dL_{G,2}^3\Reg + 2t\Reg dL_{G,2}^2
 + \frac{1}{4}\mathrm{FD}(\rho_t\|\mu_*^{i, \vartheta_t}).
\end{equation}
\item
Using the relation $\inp{a}{b}\le \|a\|_2^2 + \frac{1}{4}\|b\|_2^2$, again, it holds that
\begin{equation}\label{Eq:expr:A3}
\begin{aligned}
\mathbf{A}_3&\le \int
\left(
\left\| 
\nabla_2G_i(\vartheta_t, z_t) - 
\nabla_2G_i(\theta_k, z_t)
\right\|_2^2 
+
\frac{1}{4}\left\|
\nabla\log\left( 
\frac{\rho_t[z_t]}{\mu_*^{i, \vartheta_t}[z_t]}
\right)
\right\|_2^2\right)
\rho_t[z_t] \diff z_t\\
\le&\frac{L_{f,2}^2}{\lambda^2}\|\vartheta_t - \theta_k\|_2^2 + \frac{1}{4}\mathrm{FD}(\rho_t\|\mu_*^{i, \vartheta_t})\\
=&\frac{L_{f,2}^2t^2\eta^2}{\lambda^2}\|r_{k+1}\|_2^2 +\frac{1}{4}\mathrm{FD}(\rho_t\|\mu_*^{i, \vartheta_t}),
\end{aligned}
\end{equation}
where the second inequality is because 
\[
\|\nabla_2G_i(\theta,z) - \nabla_2G_i(\theta',z)\|_2=\frac{1}{\lambda}\|\nabla_z f_{\theta}(z) - \nabla_zf_{\theta'}(z)\|_2\le L_{f,2}/\lambda.
\]
\item
Using the relation $\inp{a}{b}\le \frac{1}{2}\|a\|_2^2 + \frac{1}{2}\|b\|_2^2$, it holds that 
\begin{equation}
\mathbf{A}_4\le \frac{\eta(L_{f,1}^2 + \|r_{k+1}\|_2^2)}{\lambda\Reg}.
\label{Eq:expr:A4}
\end{equation}
\end{itemize}
% % Define the Fisher divergence
% % \[
% % \mathrm{FD}(\rho_t\|\mu_*^{i, \vartheta_t}) = \int \rho_t\left\| 
% % \nabla\log\left( 
% % \frac{\rho_t}{\mu_*^{i, \vartheta_t}}
% % \right)
% % \right\|^2.
% % \]
% % We can see $a_1 = -\mathrm{FD}(\rho_t\|\mu_*^{i, \vartheta_t})$, and 

% Also,
% \[
% b\le \frac{\eta(L_{f,1}^2 + \|r_{k+1}\|_2^2)}{\lambda\Reg}.
% \]
As $\mu_*^{i, \vartheta_t}$ satisfies $\alpha$-LSI, 
\begin{equation}
\mathcal{D}_{\mathrm{KL}}(\rho_t\|\mu_*^{i, \vartheta_t})
\le
\frac{1}{2\alpha}
\mathrm{FD}(\rho_t\|\mu_*^{i, \vartheta_t}).\label{Eq:expr:LSI}
\end{equation}
Sustituting \eqref{Eq:expr:A2}, \eqref{Eq:expr:A3}, \eqref{Eq:expr:A4}, and \eqref{Eq:expr:LSI} into \eqref{Eq:time:div:KL}, we obtain 
%Combining those terms together, 
%we find
\begin{align*}
&\frac{\diff}{\diff t}\mathcal{D}_{\mathrm{KL}}(\rho_t \| \mu_*^{i, \vartheta_t})\\
\le&
-\alpha \mathcal{D}_{\mathrm{KL}}(\rho_t \| \mu_*^{i, \vartheta_t})
+
\frac{4L_{G,2}^4t^2}{\alpha}\mathcal{D}_{\mathrm{KL}}(\rho_0\| \mu_*^{i, \theta_k}) \\
&\qquad\qquad\qquad\qquad+ 2dL_{G,2}^3\Reg t^2 + 2\Reg dL_{G,2}^2t
+
\frac{L_{f,2}^2\eta^2t^2}{\lambda^2}\|r_{k+1}\|^2 + \frac{\eta(L_{f,1}^2 + \|r_{k+1}\|_2^2)}{\lambda\Reg}\\
\le&-\alpha \mathcal{D}_{\mathrm{KL}}(\rho_t \| \mu_*^{i, \vartheta_t})
+
\frac{4L_{G,2}^4\tau^2}{\alpha}\mathcal{D}_{\mathrm{KL}}(\rho_0\| \mu_*^{i, \theta_k}) \\
&\qquad\qquad\qquad\qquad+ 2dL_{G,2}^3\Reg\tau^2 + 2\Reg dL_{G,2}^2\tau
+
\frac{L_{f,2}^2\eta^2\tau^2}{\lambda^2}\|r_{k+1}\|^2 + \frac{\eta(L_{f,1}^2 + \|r_{k+1}\|_2^2)}{\lambda\Reg}.
% \le&-\alpha \mathcal{D}_{\mathrm{KL}}(\rho_t \| \mu_*^{i, \vartheta_t})
% +
% \frac{4\tau^2L_{V,2}^4}{\alpha}\mathcal{D}_{\mathrm{KL}}(\rho_0\| \mu_*^{i, \theta_k})
% +
% 2\tau^2\Reg dL_{V,2}^3\\
% &\qquad\qquad+ 2\tau\Reg dL_{V,2}^2
% +
% \frac{L_{f,2}^2\tau^2\eta^2}{\lambda^2}\|r_{k+1}\|^2 + \frac{\eta(L_{f,1}^2 + \|r_{k+1}\|_2^2)}{\lambda\Reg}.
\end{align*}
Since $\tau\le \frac{1}{L_{G,2}}, \eta\le 1$, and $\frac{L_{f,2}^2\tau^2}{\lambda}\le \frac{1}{\Reg}$, we further obtain
\begin{align*}
&\frac{\diff}{\diff t}\mathcal{D}_{\mathrm{KL}}(\rho_t \| \mu_*^{i, \vartheta_t})\\
\le&-\alpha \mathcal{D}_{\mathrm{KL}}(\rho_t \| \mu_*^{i, \vartheta_t})
+
\frac{4L_{G,2}^4\tau^2}{\alpha}\mathcal{D}_{\mathrm{KL}}(\rho_0\| \mu_*^{i, \theta_k}) \\
&\qquad\qquad\qquad\qquad+ 4\Reg dL_{G,2}^2\tau
+
\frac{L_{f,2}^2\eta^2\tau^2}{\lambda^2}\|r_{k+1}\|^2 + \frac{\eta(L_{f,1}^2 + \|r_{k+1}\|_2^2)}{\lambda\Reg}\\
\le&-\alpha \mathcal{D}_{\mathrm{KL}}(\rho_t \| \mu_*^{i, \vartheta_t})
+
\frac{4L_{G,2}^4\tau^2}{\alpha}\mathcal{D}_{\mathrm{KL}}(\rho_0\| \mu_*^{i, \theta_k})
+ 4\Reg dL_{G,2}^2\tau
+
 \frac{\eta(L_{f,1}^2 + 2\|r_{k+1}\|_2^2)}{\lambda\Reg}.
\end{align*}
Define constants
\begin{align*}
\mathbf{C}_1&=\frac{4L_{G,2}^4\tau^2}{\alpha}\mathcal{D}_{\mathrm{KL}}(\rho_0\| \mu_*^{i, \theta_k}),\\
\mathbf{C}_2&=4\Reg dL_{G,2}^2\tau
+
 \frac{\eta(L_{f,1}^2 + 2\|r_{k+1}\|_2^2)}{\lambda\Reg}.
\end{align*}
% Let us specify
% \[
% \tau=\frac{1}{k^{1/3}}\cdot\min\left( 
% \frac{1}{1+L_{f,2}/\lambda}, \sqrt{\frac{\lambda}{\Reg\eta L_{f,2}^2}}, \frac{1}{\alpha}, \frac{\alpha}{4(1+L_{f,2}/\lambda)^2}
% \right),\quad \eta=\frac{1}{k^{1/3}}.
% \]
% Then,
% \[
% \frac{\diff}{\diff t}\mathcal{D}_{\mathrm{KL}}(\rho_t \| \mu_*^{i, \vartheta_t})
% \le -\alpha \mathcal{D}_{\mathrm{KL}}(\rho_t \| \mu_*^{i, \vartheta_t})
% +
% \frac{4\tau^2L_{V,2}^4}{\alpha}\mathcal{D}_{\mathrm{KL}}(\rho_0\| \mu_*^{i, \theta_k})
% +4\tau\Reg dL_{V,2}^2 +
% \frac{\eta(L_{f,1}^2 + 2\|r_{k+1}\|_2^2)}{\lambda\Reg}.
% %\frac{3\eta L_{f,2}^2}{\lambda\Reg}
% \]
% Define the constant 
% \[
% C_1 = 4\tau\Reg dL_{V,2}^2 + \frac{\eta(L_{f,1}^2 + 2\|r_{k+1}\|_2^2)}{\lambda\Reg}.
% \]
By Gr{\"o}nwall’s inequality in Lemma~\ref{Grow:ineq}, we find
\begin{align*}
&\mathcal{D}_{\mathrm{KL}}(\rho_t \| \mu_*^{i, \vartheta_t})
\\\le& 
\frac{(\mathbf{C}_1+\mathbf{C}_2)e^{-\alpha t}(e^{\alpha t} - 1)}{\alpha} + 
\mathcal{D}_{\mathrm{KL}}(\rho_0\| \mu_*^{i, \theta_k})e^{-\alpha t}\\
\le&2(\mathbf{C}_1+\mathbf{C}_2)\tau e^{-\alpha\tau} + \mathcal{D}_{\mathrm{KL}}(\rho_0\| \mu_*^{i, \theta_k})e^{-\alpha t}\\
=&2\mathbf{C}_2\tau e^{-\alpha\tau} + \left( 
1 + \frac{8\tau^3L_{G,2}^4}{\alpha}
\right)\mathcal{D}_{\mathrm{KL}}(\rho_0\| \mu_*^{i, \theta_k})e^{-\alpha t}\\
\le&2\mathbf{C}_2\tau e^{-\alpha\tau} + \left( 
1 + \frac{\alpha\tau}{2}
\right)\mathcal{D}_{\mathrm{KL}}(\rho_0\| \mu_*^{i, \theta_k})e^{-\alpha\tau},
\end{align*}
where the second inequality is because $\alpha t\le \alpha\tau\le 1$ and $e^{\alpha t}\le 1+2\alpha t$ and the last inequality is because $\frac{8\tau^3L_{G,2}^4}{\alpha}\le \frac{\alpha\tau}{2}$.
By taking the time $t=\tau$, we finally obtain that
\[
\mathcal{D}_{\mathrm{KL}}(\tilde{\mu}_{k+1}^{(i)} \| \mu_*^{i, \theta_{k+1}})
\le 
2\mathbf{C}_2\tau e^{-\alpha\tau} + \left( 
1 + \frac{\alpha\tau}{2}
\right)\mathcal{D}_{\mathrm{KL}}(\mu_k^{(i)}\| \mu_*^{i, \theta_k})e^{-\alpha\tau}.
\]
Therefore, 
\begin{equation}\label{Eq:upper:lemma:2}
\begin{aligned}
\bE\big[ \mathcal{D}_{\mathrm{KL}}(\tilde{\mu}_{k+1}^{(i)} \| \mu_*^{i, \theta_{k+1}})\big| \mathcal{G}_{k}
\big]&\le \left( 
1 + \frac{\alpha\tau}{2}
\right)e^{-\alpha \tau}\mathcal{D}_{\mathrm{KL}}(\mu_k^{(i)}\| \mu_*^{i, \theta_k})
+
2\tau e^{-\alpha\tau}\left( 
4\Reg dL_{G,2}^2\tau + \frac{\eta L_{f,1}^2}{\lambda\Reg}
\right)\\
&\qquad\qquad\qquad\qquad\qquad\qquad + \frac{4\eta\tau}{\lambda\Reg}\cdot e^{-\alpha\tau}\cdot \bE\big[ \|r_{k+1}\|_2^2\big| \mathcal{G}_{k}
\big]\\
&\le \left( 
1 - \frac{\alpha\tau}{4}
\right)\mathcal{D}_{\mathrm{KL}}(\mu_k^{(i)}\| \mu_*^{i, \theta_k})
+
2\tau e^{-\alpha\tau}\left( 
4\Reg dL_{G,2}^2\tau + \frac{\eta L_{f,1}^2}{\lambda\Reg}
\right)\\
&\qquad\qquad\qquad\qquad\qquad\qquad + \frac{4\eta\tau}{\lambda\Reg}\cdot e^{-\alpha\tau}\cdot \bE\big[ \|r_{k+1}\|_2^2\big| \mathcal{G}_{k}
\big],
\end{aligned}
\end{equation}
where the second inequality is due to $\frac{\alpha\tau}{2}\le \frac{1}{2}$ and 
$\left( 
1 + \frac{\alpha\tau}{2}
\right)e^{-\alpha \tau}\le e^{-\alpha\tau/2}\le 1-\frac{\alpha\tau}{4}$.

% \begin{multline*}
% \bE\big[ \mathcal{D}_{\mathrm{KL}}(\tilde{\mu}_{k+1}^{(i)} \| \mu_*^{i, \theta_{k+1}})\big| \mathcal{G}_{k}
% \big]\le \left( 
% 1 + \frac{\alpha\tau}{2}
% \right)e^{-\alpha \tau}\mathcal{D}_{\mathrm{KL}}(\mu_k^{(i)}\| \mu_*^{i, \theta_k})\\
% +2\tau e^{-\alpha\tau}\left( 
% 4\Reg dL_{G,2}^2\tau + \frac{\eta L_{f,1}^2}{\lambda\Reg}
% \right)
% +
% \frac{4\eta\tau}{\lambda\Reg}\cdot e^{-\alpha\tau}\cdot \bE\big[ \|r_{k+1}\|_2^2\big| \mathcal{G}_{k}
% \big]
% % 8\tau^2\Reg dL_{V,2}^2e^{-\alpha\tau} + \frac{2\eta\tau L_{f,1}^2}{\lambda\Reg}e^{-\alpha\tau} + \frac{4\eta\tau}{\lambda\Reg}\cdot e^{-\alpha\tau}\cdot\bE_{I_k}\big[ \|r_{k+1}\|_2^2\big| \mathcal{G}_{k}
% % \big]\\
% % \le \left( 
% % 1 - \frac{\alpha\tau}{4}
% % \right)\mathcal{D}_{\mathrm{KL}}(\mu_k^{(i)}\| \mu_*^{i, \theta_k})+8\tau^2\Reg dL_{V,2}^2e^{-\alpha\tau} + \frac{2\eta\tau L_{f,1}^2}{\lambda\Reg}e^{-\alpha\tau} + \frac{4\eta\tau}{\lambda\Reg}\cdot e^{-\alpha\tau}\cdot\bE_{I_k}\big[ \|r_{k+1}\|_2^2\big| \mathcal{G}_{k}
% % \big]
% \end{multline*}

\paragraph{Bounding $\bE\big[ \mathcal{D}_{\mathrm{KL}}({\mu}_{k+1}^{(i)} \| \mu_*^{i, \theta_{k+1}})\big| \mathcal{G}_{k}
\big]$.}
Substituting \eqref{Eq:upper:lemma:1} and \eqref{Eq:upper:lemma:2} into \eqref{Eq:KL:mu:mu:sta}, we obtain that 
\begin{align*}
&\bE\big[ 
\mathcal{D}_{\mathrm{KL}}(\mu_{k+1}^{(i)} \| \mu_*^{i, \theta_{k+1}})\big| \mathcal{G}_{k}
\big]\\=&
\left( 
1 - \frac{|I_k|}{n}
\right)\bE\big[ \mathcal{D}_{\mathrm{KL}}(\mu_{k}^{(i)} \| \mu_*^{i, \theta_{k+1}})\big| \mathcal{G}_{k}
\big]
+
\frac{|I_k|}{n}\bE\big[ \mathcal{D}_{\mathrm{KL}}(\tilde{\mu}_{k+1}^{(i)} \| \mu_*^{i, \theta_{k+1}})\big| \mathcal{G}_{k}
\big]\\
\le&
\left( 
1 - \frac{|I_k|}{n}
\right)
\left(
\mathcal{D}_{\mathrm{KL}}(\mu_{k}^{(i)} \| \mu_*^{i, \theta_k}) + 
\frac{\eta\tau L_{f,1}^2}{\lambda\Reg} + \frac{\eta\tau}{\lambda\Reg}\bE\big[ \|r_{k+1}\|_2^2\big| \mathcal{G}_{k}
\big]
\right) 
\\
&\quad+
\frac{|I_k|}{n}\Bigg( 
\left( 
1 - \frac{\alpha\tau}{4}
\right)\mathcal{D}_{\mathrm{KL}}(\mu_k^{(i)}\| \mu_*^{i, \theta_k})+
2\tau e^{-\alpha\tau}\left( 
4\Reg dL_{G,2}^2\tau + \frac{\eta L_{f,1}^2}{\lambda\Reg}
\right)
\\
&\qquad\qquad\qquad\qquad\qquad\qquad\qquad\qquad\qquad\qquad\qquad\qquad+
\frac{4\eta\tau}{\lambda\Reg}\cdot e^{-\alpha\tau}\cdot \bE\big[ \|r_{k+1}\|_2^2\big| \mathcal{G}_{k}
\big]
\Bigg)\\
\le&\left( 
1 - \frac{\alpha\tau|I_k|}{4n}
\right)\mathcal{D}_{\mathrm{KL}}(\mu_k^{(i)}\| \mu_*^{i, \theta_k}) + \frac{3\eta\tau L_{f,1}^2}{\lambda\Reg} + \frac{8\tau^2\Reg dL_{G,2}^2|I_k|}{n} + \frac{5\eta\tau}{\lambda\Reg}\bE\big[ \|r_{k+1}\|_2^2\big| \mathcal{G}_{k}
\big].
\end{align*}
Taking the full expectation, we obtain that 
\begin{align*}
&\bE\big[ 
\mathcal{D}_{\mathrm{KL}}(\mu_{k+1}^{(i)} \| \mu_*^{i, \theta_{k+1}})
\big]\\
\le&\left( 
1 - \frac{\alpha\tau|I_k|}{4n}
\right)\bE\big[ 
\mathcal{D}_{\mathrm{KL}}(\mu_{k}^{(i)} \| \mu_*^{i, \theta_{k}})
\big] + \frac{3\eta\tau L_{f,1}^2}{\lambda\Reg} + \frac{8\tau^2\Reg dL_{G,2}^2|I_k|}{n} + \frac{5\eta\tau}{\lambda\Reg}\bE\big[ \|r_{k+1}\|_2^2
\big].
\end{align*}
Taking summation over $k=0,\ldots,T$, we have that 
\begin{align*}
&\sum_{k=0}^T\bE\big[ 
\mathcal{D}_{\mathrm{KL}}(\mu_{k}^{(i)} \| \mu_*^{i, \theta_{k}})
\big]\\
\le&\frac{4n}{\alpha\tau|I_k|}
\mathcal{D}_{\mathrm{KL}}(\mu_{0}^{(i)} \| \mu_*^{i, \theta_{0}})
 + \frac{12\eta L_{f,1}^2Tn}{\lambda\Reg\alpha|I_k|} + \frac{32\tau\Reg dL_{G,2}^2T}{\alpha} + \frac{20\eta n}{\lambda\Reg\alpha|I_k|}\sum_{k=0}^{T-1}\bE\big[ \|r_{k+1}\|_2^2
\big].
\end{align*}
The proof is complete.
\end{proof}

Before showing the proof of Theorem~\ref{Theorem:convergence:final}, we provide the following error bounds:
\begin{itemize}
    \item
Based on the expression in \eqref{Eq:F:aux}, it holds that
\begin{align*}
&\left\| \nabla F(\theta_k) - \nabla F(\theta_k; \mu_{k}^{(1:n)})\right\|_2
\le\frac{1}{n}\sum_{i\in[n]}\left\|
\bE_{\mu_{k}^{(i)} - \mu_*^{i, \theta_k}}[\nabla f_{\theta_k}(z)]
\right\|_2\\
\le&\frac{1}{n}\sum_{i\in[n]}L_{f,1}\cdot\mathrm{TV}(\mu_{k}^{(i)}, \mu_*^{i, \theta_k})\\
\le&\frac{L_{f,1}}{n}\sum_{i\in[n]}\sqrt{
\frac{1}{2}
\mathcal{D}_{\mathrm{KL}}(\mu_{k}^{(i)}, \mu_*^{i, \theta_k})
},
\end{align*}
which further implies
\[
\left\| \nabla F(\theta_k) - \nabla F(\theta_k; \mu_{k}^{(1:n)})\right\|_2^2\le 
\frac{L_{f,1}^2}{2n}\sum_{i\in[n]}\mathcal{D}_{\mathrm{KL}}(\mu_{k}^{(i)}, \mu_*^{i, \theta_k}).
\]
\item
Since $\mathbb{V}\mathrm{ar}(v_{k+1}\mid \mathcal{G}_{k-1})\le \bE[\|v_{k+1}\|_2^2\mid \mathcal{G}_{k-1}]\le L_{f,1}^2$, it holds that
\[
\bE\left[\left\|\nabla F(\theta_k; \mu_{k}^{(1:n)}) - v_{k+1}\right\|_2^2\middle|\mathcal{G}_{k-1}\right]
\le \frac{L_{f,1}^2}{|I_k|}.
\]
\item 
Substituting the error bounds above into \eqref{Eq:gradient:diff:lemma}, we obtain
\begin{equation*}
\begin{multlined}
\bE\big[\|\nabla F(\theta_k) - r_{k+1}\|_2^2\big|\mathcal{G}_{k-1}\big]
\le
(1-\beta_0)\|\nabla F(\theta_{k-1}) - r_{k}\|_2^2 
+ 
\frac{4L_{f,2}^2\tau^2\eta^2}{\beta_0}\|r_k\|_2^2 \\+ 
\frac{2\beta_0L_{f,1}^2}{n}\sum_{i\in[n]}\mathcal{D}_{\mathrm{KL}}(\mu_{k}^{(i)}, \mu_*^{i, \theta_k})
+ \frac{\beta_0^2L_{f,1}^2}{|I_k|}.
\end{multlined}%\label{Eq:gradient:diff:lemma}
\end{equation*}
Consequently, it holds that
\begin{equation}\label{Eq:before:Thm:4.5}
\begin{aligned}
&\sum_{k=0}^T\bE\big[\|\nabla F(\theta_k) - r_{k+1}\|_2^2\big]
\le\frac{1}{\beta_0}\bE\big[\|\nabla F(\theta_0) - r_{1}\|_2^2\big] \\
&\quad+ \frac{4L_{f,2}^2\tau^2\eta^2}{\beta_0^2}\sum_{k=1}^T\|r_k\|_2^2 + \frac{2L_{f,1}^2}{n}\sum_{i\in[n]}\sum_{k=1}^T\mathcal{D}_{\mathrm{KL}}(\mu_{k}^{(i)}, \mu_*^{i, \theta_k})
+
\frac{\beta_0L_{f,1}^2T}{|I_k|}.
\end{aligned}
\end{equation}
\end{itemize}

% We first show that
% \begin{multline*}
% \bE\big[\|\nabla F(\theta_k) - r_{k+1}\|_2^2\big|\mathcal{G}_{k-1}\big]
% \le
% (1-\beta_0)\|\nabla F(\theta_{k-1}) - r_{k}\|_2^2 
% + 
% \frac{4L_{f,2}^2}{\beta_0}\|\theta_k - \theta_{k-1}\|_2^2 \\+ 
% 4\beta_0\left\| \nabla F(\theta_k) - \nabla F(\theta_k; \mu_{k}^{(1:n)})\right\|_2^2 
% + \beta_0^2\bE\left[\left\|\nabla F(\theta_k; \mu_{k}^{(1:n)}) - v_{k+1}\right\|_2^2\middle|\mathcal{G}_{k-1}\right].
% \end{multline*}
% Besides, it can be shown that 

\begin{proof}[Proof of Theorem~\ref{Theorem:convergence:final}]
The average over gradient norms can be bounded as
%Combining Lemmas~\ref{Lemma:descent}, \ref{Lemma:gradient:diff}, and \ref{Lemma:KL:div}, we derive that 
\begin{align*}
&\frac{1}{T+1}\sum_{k=0}^{T}\bE\big[\|\nabla F(\theta_k)\|_2^2\big]\\
\le&\frac{2\bE[F(\theta_0) - F(\theta_{*})]}{\eta\tau T} + 
\frac{1}{T}\sum_{k=0}^T\bE\big[\|\nabla F(\theta_k) - r_{k+1}\|_2^2\big]-\frac{1}{2T}\sum_{k=0}^T\bE[\|r_{k+1}\|_2^2]\\
\le&\frac{2\bE[F(\theta_0) - F(\theta_{*})]}{\eta\tau T} + \frac{
\bE\big[\|\nabla F(\theta_0) - r_1\|_2^2\big]
}{\beta_0T} + \frac{4L_{f,2}^2\tau^2\eta^2}{\beta_0^2T}\sum_{k=0}^T\bE\big[\|r_k\|_2^2\big]\\
&\qquad\qquad + \frac{2L_{f,1}^2}{nT}\sum_{i\in[n]}\sum_{k=1}^T\bE\left[\mathcal{D}_{\mathrm{KL}}(\mu_{k}^{(i)}, \mu_*^{i, \theta_k})\right] + \frac{\beta_0L_{f,2}^2}{|I_k|}-\frac{1}{2T}\sum_{k=0}^T\bE[\|r_{k+1}\|_2^2],
\end{align*}
wjere the first inequality is based on Lemma~\ref{Lemma:descent} and we denote $\theta_*$ the optimal solution of $\min_{\theta}~F(\theta)$, and the second inequality is by \eqref{Eq:before:Thm:4.5}.

Based on Lemma~\ref{Lemma:KL:div}, we further obtain that 
\begin{align*}
&\frac{1}{T+1}\sum_{k=0}^{T}\bE\big[\|\nabla F(\theta_k)\|_2^2\big]\\
\le&\frac{2\bE[F(\theta_0) - F(\theta_{*})]}{\eta\tau T} + \frac{
\bE\big[\|\nabla F(\theta_0) - r_1\|_2^2\big]
}{\beta_0T} + \frac{4L_{f,2}^2\tau^2\eta^2}{\beta_0^2T}\sum_{k=0}^T\bE\big[\|r_k\|_2^2\big]
+ 2L_{f,2}^2\Bigg[\\
&\quad\frac{4}{\alpha\tau|I_k|T}
\sum_{i=1}^n
\mathcal{D}_{\mathrm{KL}}(\mu_{0}^{(i)} \| \mu_*^{i, \theta_{0}})
 + \frac{12\eta L_{f,1}^2n}{\lambda\Reg\alpha|I_k|} 
+
\frac{32\tau\Reg dL_{G,2}^2}{\alpha}
 + \frac{20\eta n}{\lambda\Reg\alpha |I_k|T}\sum_{t=0}^{T-1}\bE\big[ \|r_{k+1}\|_2^2
\big]\Bigg]\\
&+ \frac{\beta_0L_{f,2}^2}{|I_k|}-\frac{1}{2T}\sum_{k=0}^T\bE[\|r_{k+1}\|_2^2]\\
=&\frac{1}{T}\Bigg\{
\frac{2\bE[F(\theta_0) - F(\theta_{*})]}{\eta\tau} + \frac{
\bE\|\nabla F(\theta_0) - z_1\|_2^2
}{\beta_0} +\frac{8L_{f,2}^2}{\alpha\tau|I_k|}\sum_{i=1}^n\mathcal{D}_{\mathrm{KL}}(\mu_{0}^{(i)} \| \mu_*^{i, \theta_{0}})
\Bigg\}\\
&+
\frac{\sum_{k=0}^{T}\bE\big[ \|r_{k}\|_2^2
\big]}{T}\Bigg[
\frac{40\eta L_{f,2}^2n}{\lambda\Reg\alpha|I_k|} + \frac{4L_{f,2}^2\tau^2\eta^2}{\beta_0^2}
-\frac{1}{2}
\Bigg]\\
&+\frac{\beta_0L_{f,2}^2}{|I_k|} + \frac{64\tau\Reg dL_{G,2}^2L_{f,2}^2}{\alpha} 
+\frac{24\eta L_{f,1}^2L_{f,2}^2n}{\lambda\Reg\alpha|I_k|}.
\end{align*}
By setting
\begin{equation}
\beta_0\le \frac{\varrho^2|I_k|}{6L_{f,2}^2},\quad 
\tau\le \frac{\varrho^2\alpha}{384\Reg dL_{G,2}^2L_{f,2}^2},\quad 
\eta\le \frac{\varrho^2\lambda\Reg\alpha|I_k|}{144L_{f,1}^2L_{f,2}^2n}, \label{Eq:beta:0}
\end{equation}
it holds that
\[
\frac{\beta_0L_{f,2}^2}{|I_k|} + \frac{64\tau\Reg dL_{G,2}^2L_{f,2}^2}{\alpha} 
+\frac{24\eta L_{f,1}^2L_{f,2}^2n}{\lambda\Reg\alpha|I_k|}
\le \frac{\varrho^2}{6} + \frac{\varrho^2}{6} + \frac{\varrho^2}{6}\le \frac{\varrho^2}{2}.
\]
By setting the number of iterations
\begin{equation}
T\ge  \max\left( 
\frac{12\bE[F(\theta_0) - F(\theta_*)]}{\eta\tau\varrho^2}, 
\frac{6\bE\|\nabla F(\theta_0) - z_1\|_2^2}{\beta_0\varrho^2}, 
\frac{48L_{f,2}^2}{\alpha\tau|I_k|\varrho^2}\sum_{i=1}^n\mathcal{D}_{\mathrm{KL}}(\mu_0^{(i)}\|\mu_*^{i,\theta_0})
\right),\label{Eq:expr:T}
\end{equation}
it holds that
\[
\begin{aligned}
&\frac{1}{T}\Bigg\{
\frac{2\bE[F(\theta_0) - F(\theta_{*})]}{\eta\tau} + \frac{
\bE\|\nabla F(\theta_0) - z_1\|_2^2
}{\beta_0} +\frac{8L_{f,2}^2}{\alpha\tau|I_k|}\sum_{i=1}^n\mathcal{D}_{\mathrm{KL}}(\mu_{0}^{(i)} \| \mu_*^{i, \theta_{0}})
\Bigg\}\\\le& \frac{\varrho^2}{6} + \frac{\varrho^2}{6} + \frac{\varrho^2}{6}\le \frac{\varrho^2}{2}.
\end{aligned}
\]
By setting 
\begin{equation}
\eta\le \frac{\lambda\Reg\alpha|I_k|}{160nL_{f,2}^2}, \quad \eta\tau\le \frac{\beta_0}{4L_{f,2}},\label{Eq:expr:eta}
\end{equation}
where the latter condition is automatically satisfied for sufficiently small $\varrho$, it holds that
\[
\frac{\sum_{k=0}^{T}\bE\big[ \|r_{k}\|_2^2
\big]}{T}\Bigg[
\frac{40\eta L_{f,2}^2n}{\lambda\Reg\alpha|I_k|} + \frac{4L_{f,2}^2\tau^2\eta^2}{\beta_0^2}
-\frac{1}{2}
\Bigg]\le0.
\]
In summary, with proper choices of parameters as in \eqref{Eq:beta:0}, \eqref{Eq:expr:T}, and \eqref{Eq:expr:eta}, it holds that 
\[
\frac{1}{T+1}\sum_{k=0}^T\bE\|\nabla F(\theta_k)\|_2^2\le \varrho^2,
\]
and the output of Algorithm~\ref{Alg:solving:SDRO:single:improved} finds $\varrho$-stationary point in $T$ iterations.
\end{proof}

\section{Additional Results}
\begin{lemma}\label{Lemma:norm:2:diff}
Let us define $\rho_0:={\mu}_{k}^{(i)}$ with $x_0\sim\rho_0$, and
\[
\diff z_t = -\nabla_2G_i(\theta_k,z_0)\diff t + \sqrt{2\Reg}\diff\mathbf{B}_t,\quad 
\mathrm{Law}(z_t)=\rho_t.
\]
Then
\[
\bE_{\rho_{0,t}}\|z_t - z_0\|_2^2\le \frac{4t^2L_{G,2}^2}{\alpha}\mathcal{D}_{\mathrm{KL}}(\rho_0\| \mu_*^{i, \theta_k}) + 2t^2dL_{G,2}\Reg + 2t\Reg d.
\]
\end{lemma}
\begin{proof}[Proof of Lemma~\ref{Lemma:norm:2:diff}]
It is noteworthy that $z_t$ has the same distribution as 
\[
z_0 - t\nabla_2G_i(\theta_k,z_0) + \sqrt{2t\Reg}\zeta_0,
\]
where $\zeta_0\sim\mathcal{N}(0,\textbf{I}_d)$ is an independent random vector.
Therefore,
\begin{align*}
&\bE_{\rho_{0,t}}\|z_t - z_0\|_2^2=\bE_{\rho_{0,t}}\|t\nabla_2G_i(\theta_k,z_0)  - \sqrt{2t\Reg}\zeta_0\|_2^2\\
\le&t^2\bE_{\rho_{0}}\|\nabla_2G_i(\theta_k,z_0) \|_2^2 + 2t\Reg d\\
\le&
t^2\left[ 
\frac{4L_{G,2}^2}{\alpha}\mathcal{D}_{\mathrm{KL}}(\rho_0\| \mu_*^{i, \theta_k})
+
2dL_{G,2}\Reg
\right] + 2t\Reg d\\
\le&\frac{4t^2L_{G,2}^2}{\alpha}\mathcal{D}_{\mathrm{KL}}(\rho_0\| \mu_*^{i, \theta_k}) + 2t^2dL_{G,2}\Reg + 2t\Reg d,
\end{align*}
where the second inequality is by Lemma~\ref{Lemma:vempala2019rapid} and $L_{G,2}=1+L_{f,2}/\lambda$.
\end{proof}

\begin{lemma}[{Lemma~12 in \citep{vempala2019rapid}}]\label{Lemma:vempala2019rapid}
Suppose $\nu$ satisfies $\alpha$-LSI and $\nu[z]\propto e^{-f(z)}$, with $f(z)$ being $L$-smooth in $z\in\mathbb{R}^d$.
For any distribution $\rho$, it holds that
\[
\bE_{z\sim\rho}[\|\nabla f(z)\|_2^2]\le 
\frac{4L^2}{\alpha}\mathcal{D}_{\mathrm{KL}}(\rho\|\nu) + 2dL.
\]
\end{lemma}

\begin{lemma}[Gr{\"o}nwall’s inequality]\label{Grow:ineq}
For any given function $u:~[0,T)\to\mathbb{R}$ with $T\in(0,\infty]$, of class $\mathcal{C}^1$ satisfying the differential inequality
\[
u'\le au
\]
for some $a\in\mathbb{R}$, it also satisfies the pointwise estimate
\[
u(t)\le e^{at}u(0),\quad \forall t\in[0,T).
\]
\end{lemma}
\end{document}